\documentclass[letterpaper]{article} 
\usepackage{aaai2027}  
\usepackage[hyphens]{url}  
\usepackage{graphicx} 
\usepackage{natbib}  
\usepackage{caption} 
\usepackage{algorithm}
\usepackage{algorithmic} 
\usepackage{booktabs}
\usepackage{makecell}
\usepackage{array}
\usepackage{tabularx}
\newcolumntype{Y}{>{\raggedright\arraybackslash}X}

\usepackage{amsmath}
\usepackage{amssymb}
\usepackage{amsthm}
\usepackage{mathtools}
\usepackage{bm}

\newtheorem{theorem}{Theorem}
\newtheorem{lemma}{Lemma}
\newtheorem{corollary}{Corollary}
\newtheorem{proposition}{Proposition}
\newtheorem{assumption}{Assumption}
\newtheorem{condition}{Condition}
\theoremstyle{definition}

\theoremstyle{remark}
\newtheorem{remark}{Remark}

\usepackage{newfloat}
\usepackage{listings}
\DeclareCaptionStyle{ruled}{labelfont=normalfont,labelsep=colon,strut=off} 
\floatstyle{ruled}
\newfloat{listing}{tb}{lst}{}
\floatname{listing}{Listing}

\title{Quantifying Depth Sufficiency in Residual Neural Networks: A First-Order Criterion}

\author{
Zeyu Liu\textsuperscript{\rm 1},
Jinhao Zhang\textsuperscript{\rm 2},
Yunquan Zhang\textsuperscript{\rm 1},
Guangming Tan\textsuperscript{\rm 1},\\
Xiang Gao\textsuperscript{\rm 3},
Fangming Liu\textsuperscript{\rm 4},
Daning Cheng\textsuperscript{\rm 1}\corresponding
}

\affiliations{
\textsuperscript{\rm 1}Institute of Computing Technology, Chinese Academy of Sciences, Beijing, China\\
\textsuperscript{\rm 2}Beijing University of Posts and Telecommunications, Beijing, China\\
\textsuperscript{\rm 3}High Performance Computing Facility Research Center, Zhejiang Lab, Hangzhou, China\\
\textsuperscript{\rm 4}Huazhong University of Science and Technology, Wuhan, China\\
}

\begin{document}

\maketitle

\begin{abstract}
How can we determine whether a trained neural network is already deep enough? We study this under a fixed function-preserving residual-growth protocol specifying insertion locations, residual families, zero-output initializations, and zero-state first-order updates. We define first-order residual depth saturation as the absence of a strict local decrease from every admissible insertion. We prove residual non-degeneracy is necessary and sufficient: additional depth has first-order value exactly when conditional activation gradients have a nonzero projection onto at least one admissible residual tangent space. This boundary is shared by descent-compatible zero-state updates and invariant under regular local reparameterizations preserving that tangent space.   Under residual-signal realizability, raw activation-gradient vanishing exactly certifies saturation. Across ResNets, GPT-2-style models, and continued-pretrained Pythia checkpoints, the maximum activation-gradient norm decreases toward a low-signal regime with depth.  Function-preserving growth also achieves converged performance competitive with training from scratch. These results support activation-gradient magnitude as a conservative diagnostic of the remaining empirical first-order value of residual depth.

\end{abstract}


\section{Introduction}
 Increasing depth is a standard way to expand neural-network capacity,
yet its marginal benefit does not persist indefinitely: models often
improve as layers are added and then enter a regime in which further
depth yields little gain in test dataset. This raises a basic question for
model scaling and adaptive growth: given a trained network, how can we
determine whether additional residual depth still has useful
optimization value? Comparing independently trained models of different
depths cannot isolate this question, since their performance differences
may reflect initialization, optimization difficulty, training budget,
regularization, or finite-sample variation. Function-preserving residual
insertion \cite{chen2016net2net,wei2016networkmorphism} provides a
cleaner setting: a residual block initialized to produce globally zero
output leaves the represented function unchanged, so the marginal value
of new depth can be studied at a common initialization. Existing growth
methods, however, provide no exact criterion for when residual growth
should terminate.

Some works on normalized residual networks introduced a residual
non-degeneracy condition and showed that it is sufficient for
constructing a locally improving expanded model
\cite{cheng2026qualitative}. This leaves open a logically prior
question: does the condition exactly characterize whether any admissible
first-order residual direction remains? We study this question relative
to a fixed \emph{residual-growth protocol} that specifies, before observing the gradient signal, the
admissible insertion candidates, the parameterized residual family with
its function-preserving initialization, and a zero-state,
descent-compatible first-order update for the inserted parameters. A
candidate is operationally saturated when this update cannot produce a
strict decrease from the function-preserving initialization for any
sufficiently small positive step. Our main theorem shows that this
boundary does not depend on the optimizer chosen within the stated
first-order class.  

The tangent-space projection is the appropriate
parameterization-independent quantity, because a raw activation gradient
may contain components that the selected residual family cannot express.
For standard zero-output blocks---a feature-producing subnetwork
followed by a trainable output projection---the parameter criterion
reduces to the cross-gradient between the activation gradient and the
residual features; holding the feature parameters fixed loses no
first-order directions, since their derivative vanishes at a zero output
projection. We further introduce a checkpoint-specific
\emph{residual-signal realizability} condition, requiring only that the
current conditional activation-gradient signal lie in this tangent
space. Under realizability the projection preserves the full signal, and
raw activation-gradient vanishing becomes a necessary and sufficient
saturation certificate.  

Our experiments combine two complementary forms of evidence. First, we
measure the sample-wise activation-gradient norm across depth in ResNets
  trained on CIFAR-10, CIFAR-100, and ImageNet-100, in
Pythia checkpoints continued-pretrained on FineWeb-Edu, and in GPT-2-style
models trained from scratch on the same corpus. Because orthogonal projection cannot
increase norm, this score upper-bounds the empirical projected
residual-growth value on every fixed sample. Second, on four controlled CIFAR-10 ResNet
configurations we insert one globally zero-output residual block at a
time, comparing function-preserving
growth with training the same final-depth architectures from random
initialization to control for optimization disadvantages of the growth
procedure. Additional blocks produce realized gains at shallower depths,
and these gains disappear once the sample-wise score enters a stable
low-signal regime.

Our contributions are fourfold. First, we introduce an operational
notion of first-order residual depth saturation and prove that its
boundary is shared by all zero-state, descent-compatible first-order
updates. Second, we characterize this boundary through the
parameterization-invariant projection of the conditional activation
gradient onto the residual tangent space, and derive the cross-gradient
criterion for standard zero-output blocks. Third, we establish exact and
approximate residual-signal realizability conditions under which raw
activation gradients can replace the projected criterion. Finally, we combine broad
fixed-sample gradient measurements with controlled one-block growth and
from-scratch comparisons, supporting low sample-wise activation-gradient
energy as a conservative practical indicator that little empirical
first-order value of residual depth remains.

\section{Related Work}

Net2Net and Network Morphism introduced function-preserving operators that
widen or deepen a trained network without changing the represented function
\cite{chen2016net2net,wei2016networkmorphism}. This idea underlies efficient
Transformer pre-training through progressive stacking, parameter reuse and
knowledge inheritance, learned or lossless expansion operators, and
variance-transfer initialization
\cite{gong2019progressive,gu2021compoundgrow,du2024stacking,chen2022bert2bert,qin2022knowledge,wang2023ligo,wang2024lemon,yuan2023incremental}.
These works address how to grow efficiently and presuppose that added
capacity is useful; none gives an exact criterion for when residual growth
stops having local optimization value.

Deciding when and where to grow is a classical theme of constructive
learning \cite{fahlman1990cascade,ash1989dynamic}. Recent policies schedule
depth growth via validation heuristics or fitting risk
\cite{wen2020autogrow,wu2024whentogrow}, and local scores select beneficial
insertions through splitting directions, gradient norms, expressivity
bottlenecks, topological derivatives, and natural-expansion scores
\cite{wu2019splitting,wu2020firefly,evci2022gradmax,verbockhaven2024growing,krishnanunni2025topological,mitchell2023selfexpanding}.
These criteria are sufficient conditions or heuristics for beneficial
growth; we instead characterize exactly when no admissible
function-preserving residual insertion admits a first-order improvement.

Residual networks behave like ensembles of shallow paths, and stochastic
depth and layer pruning of large language models reveal substantial
redundancy in trained depth
\cite{veit2016ensembles,huang2016stochastic,gromov2025unreasonable,men2025shortgpt};
zero-initialized residual branches are benign and trainable
\cite{zhang2019fixup,bachlechner2021rezero}, supporting our zero-output
initialization. Unlike scaling-law comparisons of independently trained
models \cite{kaplan2020scaling,levine2020limits}, our function-preserving
setting isolates the marginal first-order value of additional depth at a
common checkpoint.

Closest to our work,
\citeauthor{cheng2026qualitative}~\shortcite{cheng2026qualitative} show that
residual non-degeneracy (Condition~1) is sufficient for constructing a
locally improving expanded model; we prove that it is also necessary under a
fixed zero-state protocol, yielding a saturation boundary invariant to the
first-order optimizer and to reparameterizations preserving the residual
tangent space.
 

\section{Problem Setup, Assumptions, and Notation}
\label{sec:preliminaries}

\subsection{Problem Setup}
\label{sec:problem-setup}

Let $\mathcal{X}$ be the input space, let $\mathcal{Y}$ be the label
space, and let $\mathcal{D}$ be a probability distribution on
$\mathcal{X}\times\mathcal{Y}$. We write $(x,y)\sim\mathcal{D}$,
where $x$ is an input and $y$ is its label. Every predictor considered
below maps $\mathcal{X}$ to $\mathbb{R}^{d_{\mathrm{out}}}$, and the
loss is a measurable function
$\ell:\mathbb{R}^{d_{\mathrm{out}}}\times\mathcal{Y}\to[0,\infty)$.
For any measurable predictor $f$, define
$R(f):=\mathbb{E}_{(x,y)\sim\mathcal{D}}[\ell(f(x),y)]$ whenever the
expectation is finite. For a fixed sample
$S=\{(x_i,y_i)\}_{i=1}^{M}$, define
$\mathcal{L}_{S}(f):=M^{-1}\sum_{i=1}^{M}\ell(f(x_i),y_i)$.

We follow the residual-insertion notation of
\cite{cheng2026qualitative}. Let $f_{\mathrm{old}}^{*}$ be a trained
reference model with $R(f_{\mathrm{old}}^{*})<\infty$. Before observing
any activation-gradient signal or growth outcome, we fix a
\emph{residual-growth protocol}. The protocol specifies a finite set
$\mathcal{I}$ of admissible insertion candidates. Each candidate
$l\in\mathcal{I}$ consists of an insertion location, a parameterized
residual family, a designated function-preserving initialization, and a
first-order optimizer direction rule for the newly inserted parameters.

At candidate $l$, decompose the reference model as
$f_{\mathrm{old}}^{*}=f_{\mathrm{top}}^{(l)}\circ
f_{\mathrm{bot}}^{(l)}$ and define
$z_l:=f_{\mathrm{bot}}^{(l)}(x)\in\mathbb{R}^{N_l}$, where $N_l$ is
the hidden-state dimension at candidate $l$. Let $\mu_l$ denote the
distribution of $z_l$ induced by $(x,y)\sim\mathcal{D}$. When a single
candidate is fixed, we suppress $l$ and write
$f_{\mathrm{top}}$, $f_{\mathrm{bot}}$, $z$, and $\mu$.

The parameterized residual family at candidate $l$ is
$\mathcal{F}_{\mathrm{res}}^{(l)}
:=\{h_{l,\theta_l}:\theta_l\in\Theta_l\}$, where
$\Theta_l\subseteq\mathbb{R}^{p_l}$ and
$0\in\operatorname{int}(\Theta_l)$. The designated local origin
satisfies $h_{l,0}(z)=0$ for every $z\in\mathbb{R}^{N_l}$; equivalently,
$h_{l,0}\equiv0$ on the full ambient hidden-state space. The notation
$\theta_l=0$ denotes a local coordinate centered at this initialization
and does not require every raw parameter in the residual branch to be
numerically zero.

The expanded model is
$f_{l,\theta_l}(x):=f_{\mathrm{top}}^{(l)}
\bigl(z_l+h_{l,\theta_l}(z_l)\bigr)$. Since $h_{l,0}\equiv0$, the
insertion preserves the represented function pointwise and
$f_{l,0}=f_{\mathrm{old}}^{*}$. Define the population and empirical
objectives associated with candidate $l$ by
$\Phi_l(\theta_l):=R(f_{l,\theta_l})$ and
$\Phi_{S,l}(\theta_l):=\mathcal{L}_{S}(f_{l,\theta_l})$.

All components of the residual-growth protocol are fixed in advance.
The candidates, residual families, parameterizations, designated
origins, and optimizer direction rules may not be changed after the
activation-gradient signal has been observed. All internal
hyperparameters that determine each direction rule are fixed; the
positive scalar step size used in the local analysis remains free.  


\subsection{Assumptions}
\label{sec:assumptions}

The assumptions below have distinct roles.
Assumptions~\ref{ass:first-order-regularity}
and~\ref{ass:first-order-optimizer} support the main saturation
theorem. Assumption~\ref{ass:tangent-completeness} is invoked only when
the exact projected criterion is replaced by the raw
activation-gradient criterion. Assumption~\ref{ass:independent-probe}
is used only for the supplementary finite-sample consistency result.

\begin{assumption}[First-order regularity]
\label{ass:first-order-regularity}
For every candidate $l\in\mathcal{I}$ and for
$\mathcal{D}$-almost every $(x,y)$, the map
$z\mapsto\ell(f_{\mathrm{top}}^{(l)}(z),y)$ is Fr\'echet differentiable
at $z=z_l$. Its measurable gradient is denoted by
$q_l(z_l,y):=\nabla_{z_l}\ell(f_{\mathrm{top}}^{(l)}(z_l),y)
\in\mathbb{R}^{N_l}$. Define the conditional population signal by
$m_l(z):=\mathbb{E}[q_l(z_l,y)\mid z_l=z]$.

For $\mu_l$-almost every $z$, the map
$\theta_l\mapsto h_{l,\theta_l}(z)$ is Fr\'echet differentiable at
$\theta_l=0$. Its measurable Jacobian with respect to $\theta_l$,
evaluated at the origin, is denoted by
$J_l(z):=J_{\theta_l}h_{l,0}(z)
\in\mathbb{R}^{N_l\times p_l}$, where $J_{\theta_l}$ denotes the
Jacobian with respect to the active residual coordinate $\theta_l$.

There exists a neighborhood of the origin on which $\Phi_l$ is finite.
The map $\Phi_l$ is Fr\'echet differentiable at $0$, and its derivative
is obtained by interchanging sample-wise differentiation and
expectation. For a fixed sample $S$, empirical statements are understood
on samples for which all corresponding sample-wise derivatives exist;
an i.i.d. sample has this property almost surely under the preceding
conditions.

We further assume
$\mathbb{E}\|q_l(z_l,y)\|_2^2<\infty$ and
$\mathbb{E}\|J_l(z_l)\|_{\sigma}^2<\infty$, where
$\|\cdot\|_{\sigma}$ is the operator norm. Conditional Jensen's
inequality then gives
$m_l\in L_2(\mu_l;\mathbb{R}^{N_l})$, and Cauchy--Schwarz ensures that
the population residual-parameter gradient defined below is finite.
\end{assumption}

 A standard sufficient condition for the interchange in
Assumption~\ref{ass:first-order-regularity} is a local integrable
Lipschitz envelope: for each candidate $l$, there exist $\rho_l>0$ and
an integrable random variable $B_l(x,y)$ such that
$|\ell(f_{l,\theta}(x),y)-\ell(f_{l,0}(x),y)|
\le B_l(x,y)\|\theta\|_2$ whenever $\|\theta\|_2\le\rho_l$.
For any sequence $\theta_k\to0$, sample-wise Fr\'echet differentiability
makes the normalized remainder converge pointwise to zero. The local
Lipschitz bound controls both the difference quotient and the norm of
its derivative by $B_l$, so the normalized remainder is dominated by
$2B_l$. Dominated convergence then yields an $o(\|\theta_k\|_2)$
population remainder, which gives Fr\'echet differentiability of
$\Phi_l$ and the stated differentiation--expectation interchange.
Uniform integrability of the local difference quotients is an
alternative sufficient condition.

\paragraph{First-order residual gradients and Condition~1.}
For a fixed sample $S$, let
$z_{l,i}:=f_{\mathrm{bot}}^{(l)}(x_i)$ and
$q_{l,i}:=q_l(z_{l,i},y_i)$. Define
$g_{\mathrm{pop}}^{(l)}
:=\mathbb{E}[J_l(z_l)^{\top}q_l(z_l,y)]$ and
$g_S^{(l)}
:=M^{-1}\sum_{i=1}^{M}J_l(z_{l,i})^{\top}q_{l,i}$.
For a direction $u\in\mathbb{R}^{p_l}$, write
$D\Phi_l(0)[u]$ and $D\Phi_{S,l}(0)[u]$ for the Fr\'echet directional
derivatives at the origin.

\begin{condition}[Residual non-degeneracy]
\label{cond:residual-nondegeneracy}
At candidate $l$, population residual non-degeneracy holds if there
exists $v_{\mathrm{pop}}^{(l)}\in\mathbb{R}^{p_l}$ such that
$D\Phi_l(0)[v_{\mathrm{pop}}^{(l)}]<0$. For a fixed sample $S$, its
empirical counterpart holds if there exists
$v_S^{(l)}\in\mathbb{R}^{p_l}$ such that
$D\Phi_{S,l}(0)[v_S^{(l)}]<0$.
\end{condition}

\begin{assumption}[Fixed zero-state descent-compatible first-order update]
\label{ass:first-order-optimizer}
At every candidate $l\in\mathcal{I}$, the optimizer direction rule and
all internal hyperparameters that determine it are fixed as part of the
residual-growth protocol. We call the optimizer \emph{zero-state} when
the active residual coordinate is initialized at $\theta_l=0$ and every
optimizer state variable associated with $\theta_l$ is initialized at
its neutral zero value. We assume that the optimizer used by the
protocol is zero-state.

Given the exact population or full-sample empirical residual-parameter
gradient $g\in\mathbb{R}^{p_l}$, the local update has the form
$\theta_l^{+}=\eta d_l(g)$, where $\eta>0$ is a free scalar step size
and $d_l:\mathbb{R}^{p_l}\to\mathbb{R}^{p_l}$ is the fixed direction
map. We assume $d_l(0)=0$ and $g^{\top}d_l(g)<0$ for every $g\neq0$.
Since $0\in\operatorname{int}(\Theta_l)$, every such direction is
feasible for all sufficiently small positive $\eta$.

This class includes gradient descent and positive-definite
preconditioned gradient descent. It also includes the first
bias-corrected Adam or AdamW update at zero state: coordinatewise,
$d_l(g)_j=-g_j/(|g_j|+\epsilon_{\mathrm{Adam}})$, where
$\epsilon_{\mathrm{Adam}}>0$ is the numerical-stability constant.
Hence $d_l(0)=0$ and $g^{\top}d_l(g)<0$ for $g\neq0$; decoupled weight
decay contributes no
first-step drift because the active residual coordinate is zero. The
analysis does not cover Hessian-based updates, externally injected
perturbations, nonzero initial optimizer states, or stochastic escape
from a zero full gradient.
\end{assumption}

\paragraph{Population and empirical first-order saturation.}
A candidate $l$ is \emph{population first-order improvable} if there
exists $\bar\eta_l>0$ such that, for every
$\eta\in(0,\bar\eta_l)$,
$\Phi_l(\eta d_l(g_{\mathrm{pop}}^{(l)}))<\Phi_l(0)$. It is
\emph{population first-order saturated} otherwise. The empirical notions
are defined by replacing $\Phi_l$ and $g_{\mathrm{pop}}^{(l)}$ with
$\Phi_{S,l}$ and $g_S^{(l)}$; the corresponding step-size threshold is
denoted by $\bar\eta_{S,l}$. The reference model is
\emph{first-order depth-saturated} relative to the fixed protocol if
every candidate is population first-order saturated. Unless explicitly
qualified, ``first-order saturated'' refers to the population notion.

\paragraph{Residual tangent-space objects.}
For each candidate $l$, let
$\mathcal{H}_l:=L_2(\mu_l;\mathbb{R}^{N_l})$ with inner product
$\langle a,b\rangle_{\mathcal{H}_l}
:=\mathbb{E}_{z_l\sim\mu_l}[a(z_l)^{\top}b(z_l)]$ and norm
$\|a\|_{\mathcal{H}_l}:=\langle a,a\rangle_{\mathcal{H}_l}^{1/2}$.
Define the residual tangent operator
$A_l:\mathbb{R}^{p_l}\to\mathcal{H}_l$ by
$(A_lu)(z):=J_l(z)u$. Since $\operatorname{Range}(A_l)$ is
finite-dimensional, it is closed; define
$\mathcal{T}_h^{(l)}:=\operatorname{Range}(A_l)$ and let $\Pi_l$ be the
orthogonal projector onto $\mathcal{T}_h^{(l)}$.

For the fixed sample $S$, define the empirical Hilbert space
$\mathcal{H}_{S,l}:=(\mathbb{R}^{N_l})^M$ with
$\langle a,b\rangle_{S,l}
:=M^{-1}\sum_{i=1}^{M}a_i^{\top}b_i$ and
$\|a\|_{S,l}:=\langle a,a\rangle_{S,l}^{1/2}$.
Define
$q_{S,l}:=(q_{l,1},\ldots,q_{l,M})\in\mathcal{H}_{S,l}$ and its matrix
representation
$Q_S^{(l)}:=[q_{l,1},\ldots,q_{l,M}]
\in\mathbb{R}^{N_l\times M}$.
The stacked tangent operator
$A_{S,l}:\mathbb{R}^{p_l}\to\mathcal{H}_{S,l}$ is
$A_{S,l}u:=(J_l(z_{l,1})u,\ldots,J_l(z_{l,M})u)$, and $\Pi_{S,l}$
denotes the orthogonal projector onto $\operatorname{Range}(A_{S,l})$.

\begin{assumption}[Residual-signal realizability]
\label{ass:tangent-completeness}
At every candidate to which the raw activation-gradient criterion is
applied, the current conditional activation-gradient signal is
realizable by the fixed residual tangent space:
$m_l\in\mathcal{T}_h^{(l)}$.

The stronger equality $\mathcal{T}_h^{(l)}=\mathcal{H}_l$ is called
\emph{population tangent completeness}. Since
$\dim\mathcal{T}_h^{(l)}\le p_l<\infty$, exact population tangent
completeness is impossible whenever $\mathcal{H}_l$ is
infinite-dimensional, as is typical for non-atomic hidden-state
distributions.

On the fixed sample $S$, residual-signal realizability means
$q_{S,l}\in\operatorname{Range}(A_{S,l})$. The stronger condition that
$A_{S,l}$ be surjective onto $\mathcal{H}_{S,l}$ is called
\emph{sample-wise tangent completeness}.
\end{assumption}

\begin{assumption}[Independent finite-sample probing]
\label{ass:independent-probe}
For the supplementary finite-sample consistency result, all
probabilities are conditional on the fixed reference model and the
complete residual-growth protocol. Let
$S_n^{\mathrm{probe}}
:=\{(\widetilde x_i,\widetilde y_i)\}_{i=1}^{n}
\overset{\mathrm{i.i.d.}}{\sim}\mathcal{D}^{n}$ be independent of the
data and randomness used to train $f_{\mathrm{old}}^{*}$. In that
result, all empirical quantities are instantiated with
$S=S_n^{\mathrm{probe}}$ and $M=n$.

For every $l\in\mathcal{I}$, assume
$\mathbb{E}\|J_l(z_l)^{\top}q_l(z_l,y)\|_2^2<\infty$. As $n$
increases, the reference model and the complete protocol remain fixed;
adaptive selection of a new candidate or residual family after observing
the probe sample is not covered.
\end{assumption}

\paragraph{Practical motivation.}
We focus on first-order optimization because gradient-based methods are
the practical default for modern large-scale neural networks, whereas
exact and structured second-order methods introduce substantially larger
memory and computational costs \cite{anil2020scalable}. Residual-signal
realizability is motivated by common ResNet and Transformer branches,
which typically consist of a feature-producing subnetwork followed by a
trainable terminal output map \cite{he2016resnet,vaswani2017attention}.
Zero-initializing that terminal map preserves the reference function
while retaining nonzero upstream features. An unconstrained output
projection alone does not guarantee realizability, however: the
upstream features must also span the sample-dependent variation required
by the current activation-gradient signal.

\subsection{Notation}
\label{sec:additional-notation}

\paragraph{Residual growth values.}
Define the local population and empirical residual growth values by
$\mathcal{V}_l:=\|\Pi_lm_l\|_{\mathcal{H}_l}$ and
$\mathcal{V}_{S,l}:=\|\Pi_{S,l}q_{S,l}\|_{S,l}$. Define their
depth-wide counterparts by
$\mathcal{V}_{\mathrm{depth}}:=\max_{l\in\mathcal{I}}\mathcal{V}_l$
and
$\mathcal{V}_{S,\mathrm{depth}}
:=\max_{l\in\mathcal{I}}\mathcal{V}_{S,l}$.
Their dependence on the fixed reference model and residual-growth
protocol is suppressed throughout.

\paragraph{Standard zero-output residual blocks.}
At candidate $l$, consider
$h_{l,U_l,V_l}(z):=V_l\psi_{l,U_l}(z)$, where
$\psi_{l,U_l}(z)\in\mathbb{R}^{r_l}$ and
$V_l\in\mathbb{R}^{N_l\times r_l}$. During the local insertion test,
$U_l=U_{0,l}$ is fixed and $V_l$ is the active parameter, initialized at
$V_l=0$. We identify
$\theta_l=\operatorname{vec}(V_l)$,
$p_l=N_lr_l$, and $\Theta_l=\mathbb{R}^{p_l}$.
Writing $\psi_l(z):=\psi_{l,U_{0,l}}(z)$, every matrix perturbation
$\Delta V_l$ satisfies
$J_l(z)\operatorname{vec}(\Delta V_l)=\Delta V_l\psi_l(z)$.
Allowing $U_l$ to vary as an additional active coordinate does not
enlarge the first-order tangent space at $V_l=0$, because
$D_{U_l}h_{l,U_l,V_l}(z)[\Delta U_l]
=V_lD_{U_l}\psi_{l,U_l}(z)[\Delta U_l]=0$ at the designated origin.
Thus fixing $U_l=U_{0,l}$ is a lossless first-order simplification.

Define the Frobenius inner product and norm by
$\langle B,C\rangle_F:=\operatorname{tr}(B^{\top}C)$ and
$\|B\|_F:=\langle B,B\rangle_F^{1/2}$. Define
$C_{\mathrm{pop}}^{(l)}
:=\mathbb{E}[q_l(z_l,y)\psi_l(z_l)^{\top}]$ and
$C_S^{(l)}:=M^{-1}\sum_{i=1}^{M}q_{l,i}\psi_l(z_{l,i})^{\top}$.
Let
$\Psi_S^{(l)}:=[\psi_l(z_{l,1}),\ldots,\psi_l(z_{l,M})]
\in\mathbb{R}^{r_l\times M}$, so that
$C_S^{(l)}=M^{-1}Q_S^{(l)}(\Psi_S^{(l)})^{\top}$.

Let
$\Pi_{\Psi,l}
:=(\Psi_S^{(l)})^{\top}
[\Psi_S^{(l)}(\Psi_S^{(l)})^{\top}]^{\dagger}\Psi_S^{(l)}$
be the orthogonal projector onto the row space of $\Psi_S^{(l)}$, where
${}^{\dagger}$ is the Moore--Penrose pseudoinverse. Under the
 tuple--matrix identification of $\mathcal H_{S,l}$, the matrix
representation of $\Pi_{S,l}q_{S,l}$ is
$Q_S^{(l)}\Pi_{\Psi,l}$. Consequently, finite-sample
residual-signal realizability is equivalent to
$Q_S^{(l)}=Q_S^{(l)}\Pi_{\Psi,l}$. Sample-wise tangent completeness is
equivalent to $\operatorname{rank}(\Psi_S^{(l)})=M$, in which case
$\Pi_{\Psi,l}=I_M$.

When one candidate is fixed, we suppress $l$ and write
$C_{\mathrm{pop}}$, $C_S$, $Q_S$, $\Psi_S$, and $\Pi_{\Psi}$.

\section{First-Order Residual Depth Saturation}
\label{sec:first-order-depth-saturation}

This section characterizes the exact boundary between residual depth
that retains first-order optimization value and residual depth that
cannot be activated from a function-preserving initialization. 

\subsection{Main Theorem: An Exact Saturation Boundary}
\label{sec:main-saturation-theorem}

Condition~\ref{cond:residual-nondegeneracy} was introduced in prior work
as a sufficient condition for a locally improving residual insertion
\cite{cheng2026qualitative}. We show that, under the fixed protocol, it
is also necessary for first-order improvement.

\begin{theorem}[Necessary and sufficient characterization of
first-order residual depth saturation]
\label{thm:first-order-saturation}
Suppose Assumptions~\ref{ass:first-order-regularity}
and~\ref{ass:first-order-optimizer} hold. At every candidate
$l\in\mathcal{I}$, the following statements are equivalent:
\begin{enumerate}
    \item Condition~\ref{cond:residual-nondegeneracy} holds at $l$.
    \item $g_{\mathrm{pop}}^{(l)}\neq0$.
    \item $\Pi_lm_l\neq0$.
    \item $\mathcal{V}_l>0$.
    \item There exists $\bar\eta_l>0$ such that, for every
    $\eta\in(0,\bar\eta_l)$,
    $\Phi_l(\eta d_l(g_{\mathrm{pop}}^{(l)}))<\Phi_l(0)$.
\end{enumerate}
Consequently,
$\mathcal{V}_{\mathrm{depth}}=0$ if and only if every candidate is
population first-order saturated, equivalently, if and only if the
reference model is first-order depth-saturated relative to the fixed
protocol.

For every fixed sample $S$ satisfying the differentiability conditions
in Assumption~\ref{ass:first-order-regularity}, the empirical
counterparts are also equivalent:
empirical Condition~\ref{cond:residual-nondegeneracy},
$g_S^{(l)}\neq0$, $\Pi_{S,l}q_{S,l}\neq0$,
$\mathcal{V}_{S,l}>0$, and the existence of
$\bar\eta_{S,l}>0$ such that
$\Phi_{S,l}(\eta d_l(g_S^{(l)}))<\Phi_{S,l}(0)$ for every
$\eta\in(0,\bar\eta_{S,l})$.
\end{theorem}

Items 1--4 characterize the existence of an arbitrary strict
first-order descent direction in the fixed residual parameter space.
Item 5 is the optimizer-realization statement: it concerns specifically
the update obtained by applying the fixed direction map $d_l$ to the
exact residual gradient. Its reverse implication uses the zero-state
property $d_l(0)=0$ and should not be read as a claim that a stationary
point cannot be left by stochastic or higher-order mechanisms.

The qualifier \emph{first-order} is essential. The theorem does not
exclude improvements obtained through higher-order curvature,
externally injected perturbations, stochastic escape, or a different
residual-growth protocol. The projected activation gradient and the
residual-parameter gradient share the same zero-versus-nonzero boundary,
but their norms generally differ. The growth value depends only on the
residual tangent subspace and is invariant under regular local
reparameterizations that preserve this subspace, whereas the Euclidean
parameter-gradient norm is coordinate-dependent.

A finite joint-insertion extension, under an additional joint
Fr\'echet-regularity condition, is stated in
 the supplementary material. A
separate local-smoothness corollary gives an explicit one-step decrease
for ordinary gradient descent without changing the main equivalence.

\subsection{Residual-Signal Realizability in Common Architectures}
\label{sec:common-tangent-completeness}

Theorem~\ref{thm:first-order-saturation} shows that the universal
criterion is the activation-gradient component contained in the fixed
residual tangent space. Residual-signal realizability permits this
projected criterion to be replaced by the raw activation-gradient
signal.

\begin{theorem}[Activation-gradient characterization under
residual-signal realizability]
\label{thm:activation-gradient-certificate}
Suppose Assumptions~\ref{ass:first-order-regularity}
and~\ref{ass:first-order-optimizer} hold at a fixed candidate $l$.
\begin{enumerate}
    \item If $m_l\in\mathcal{T}_h^{(l)}$, then
    $\mathcal{V}_l=\|m_l\|_{\mathcal{H}_l}$. Consequently,
    Condition~\ref{cond:residual-nondegeneracy} holds if and only if
    $m_l\neq0$, and the candidate is population first-order saturated
    if and only if $m_l=0$ $\mu_l$-almost everywhere.

    \item If $q_{S,l}\in\operatorname{Range}(A_{S,l})$, then
    $\mathcal{V}_{S,l}=\|q_{S,l}\|_{S,l}$. Consequently, empirical
    Condition~\ref{cond:residual-nondegeneracy} holds if and only if
    $q_{S,l}\neq0$, equivalently $Q_S^{(l)}\neq0$.

    \item Independently of Assumption~\ref{ass:tangent-completeness}, population
    Condition~\ref{cond:residual-nondegeneracy} holds if and only if
    $C_{\mathrm{pop}}^{(l)}\neq0$, and its empirical counterpart holds
    if and only if $C_S^{(l)}\neq0$.

    \item For this standard block, sample-wise tangent completeness is
    equivalent to $\operatorname{rank}(\Psi_S^{(l)})=M$. Under this
    condition, $\Pi_{\Psi,l}=I_M$, and therefore
    $C_S^{(l)}=0$ if and only if $Q_S^{(l)}=0$.
\end{enumerate}
If residual-signal realizability holds at every candidate, simultaneous
activation-gradient vanishing is equivalent to depth-wide first-order
saturation.
\end{theorem}

Activation-gradient vanishing is always sufficient for local
first-order saturation, even without realizability. Realizability is
needed only for the reverse implication. More quantitatively, let
$\varepsilon_{\mathrm{real}}\in[0,1)$. If $m_l\neq0$ and
$\operatorname{dist}_{\mathcal H_l}(m_l,\mathcal{T}_h^{(l)})
\le\varepsilon_{\mathrm{real}}\|m_l\|_{\mathcal{H}_l}$, then
Pythagoras gives
$\mathcal{V}_l\ge\sqrt{1-\varepsilon_{\mathrm{real}}^2}
\|m_l\|_{\mathcal{H}_l}>0$.
The same statement holds on a fixed sample after replacing
$m_l$, $\mathcal T_h^{(l)}$, and $\|\cdot\|_{\mathcal H_l}$ by
$q_{S,l}$, $\operatorname{Range}(A_{S,l})$, and
$\|\cdot\|_{S,l}$, respectively.

Many ResNet and Transformer branches contain a feature-producing
subnetwork followed by a trainable output map. Zero-initializing the
terminal map preserves the reference function while keeping the
upstream features nonzero. Wide and diverse features can therefore make
signal-relative realizability plausible. 

For an arbitrary fixed residual block, the projected activation gradient
is the parameterization-invariant exact criterion; the
residual-parameter gradient has the same zero-versus-nonzero boundary.
For a standard zero-output block, the latter reduces to the
activation--feature cross-gradient. Under residual-signal realizability,
the raw activation gradient may be used directly.

\subsection{Finite-Sample Certification}
\label{sec:finite-sample-certification}

The preceding results characterize saturation through population
residual-parameter gradients, which are not directly observable. We
therefore consider an independent probe sample
$S_n^{\mathrm{probe}}
=\{(\widetilde{x}_i,\widetilde{y}_i)\}_{i=1}^{n}$
drawn after the reference model and the complete residual-growth
protocol have been fixed.

For each candidate $l\in\mathcal{I}$, let
$\xi_l(x,y):=J_l(z_l)^{\top}q_l(z_l,y)$,
$g_{\mathrm{pop}}^{(l)}:=\mathbb{E}[\xi_l]$, and $\widehat{g}_{n}^{(l)}
:=
\frac{1}{n}
\sum_{i=1}^{n}
\xi_l(\widetilde{x}_i,\widetilde{y}_i)$.
Define
$\Gamma:=\max_{l\in\mathcal{I}}
\lVert g_{\mathrm{pop}}^{(l)}\rVert_2$
and
$\widehat{\Gamma}_n
:=\max_{l\in\mathcal{I}}
\lVert\widehat{g}_n^{(l)}\rVert_2$.

\begin{proposition}[Finite-sample error of the saturation score]
\label{prop:finite-sample-saturation}

Suppose Assumption~\ref{ass:independent-probe} holds, and let
$\sigma_l^2
:=\mathbb{E}
\lVert\xi_l-g_{\mathrm{pop}}^{(l)}\rVert_2^2$.
Then, for every $t>0$,
\begin{equation}
\Pr\left(
    \left|
        \widehat{\Gamma}_n-\Gamma
    \right|
    \ge t
\right)
\le
\frac{1}{nt^2}
\sum_{l\in\mathcal{I}}\sigma_l^2.
\label{eq:finite-sample-score-bound}
\end{equation}

Consequently, fix $\delta\ge0$ and consider the separated hypotheses
$\Gamma\le\delta$ and $\Gamma\ge\delta+2t$.
The decision rule that declares saturation when
$\widehat{\Gamma}_n\le\delta+t$
has error probability at most the right-hand side of
\eqref{eq:finite-sample-score-bound}.
\end{proposition}



\section{Experiments}
\label{sec:experiments}
We conduct two complementary experiments to evaluate the proposed depth-saturation criterion. First, we examine how the maximum per-example activation-gradient norm changes with depth across ResNets trained on CIFAR-10, CIFAR-100, and ImageNet-100, GPT-2-style models trained on FineWeb-Edu, and official Pythia checkpoints continued-pretrained on the same corpus.   Second, we compare function-preserving growth with training the same final architectures from random initialization to determine whether zero-output insertion impairs converged solution quality. The training and model settings are provided in the Appendix. All models are well-trained.

\subsection{Activation-Gradient Norms Across Model Depth}
\label{sec:gradient-depth-scaling}

\begin{figure*}[t]
    \centering

    \begin{minipage}[t]{0.32\textwidth}
        \centering
        \includegraphics[width=\linewidth]{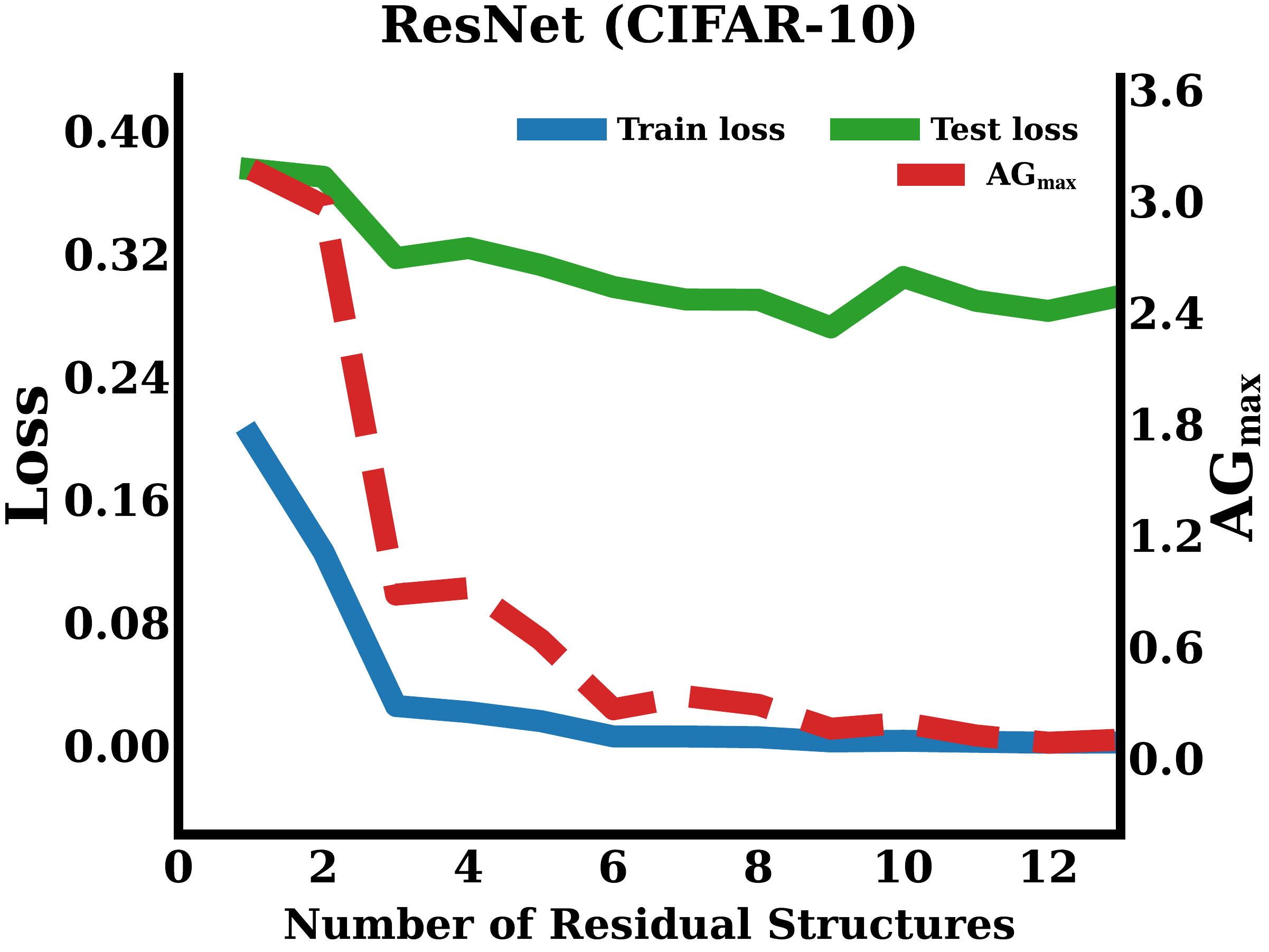}
        \par\smallskip
    \end{minipage}
    \hfill
    \begin{minipage}[t]{0.32\textwidth}
        \centering
        \includegraphics[width=\linewidth]{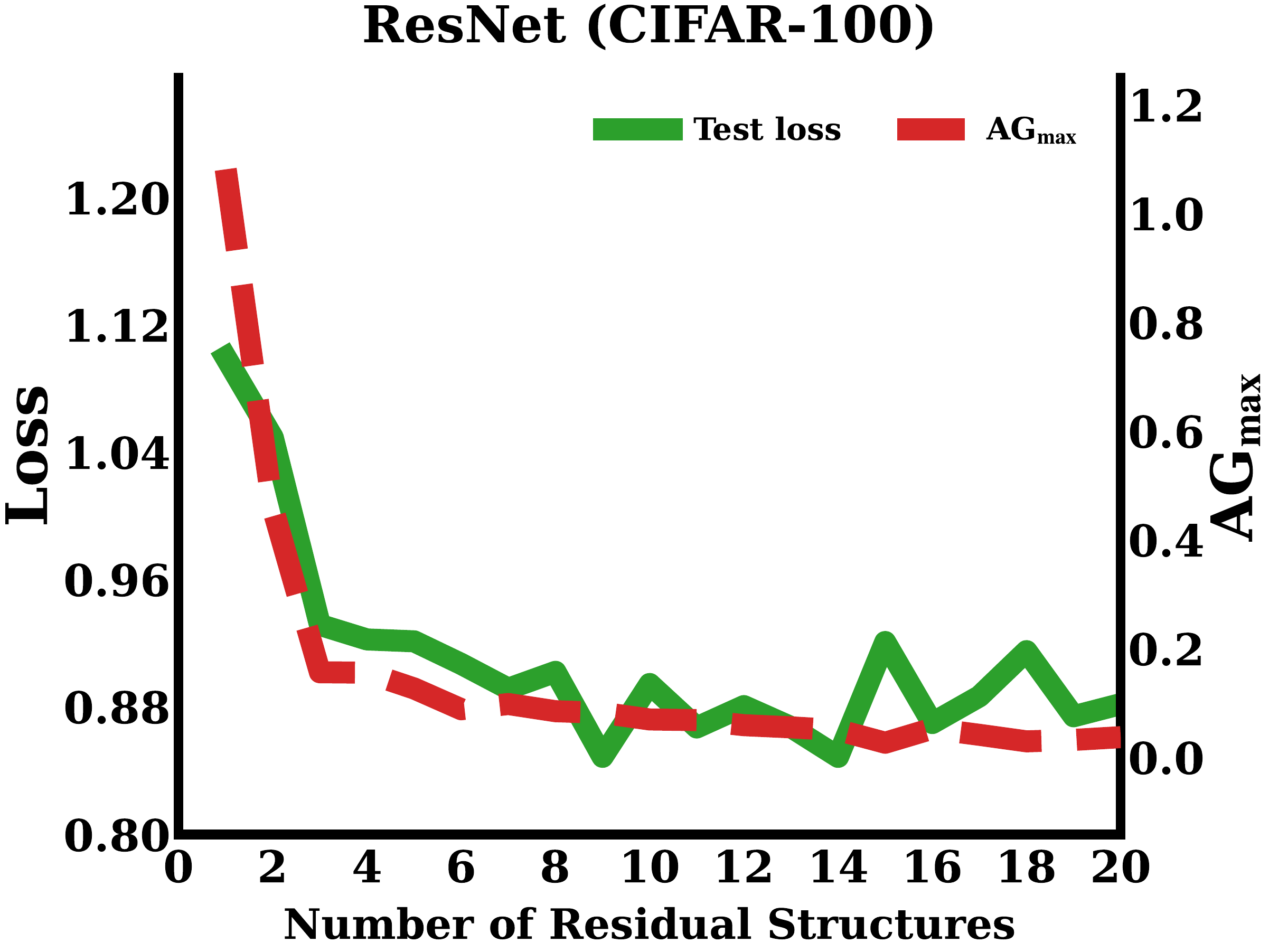}
        \par\smallskip
    \end{minipage}
    \hfill
    \begin{minipage}[t]{0.32\textwidth}
        \centering
        \includegraphics[width=\linewidth]{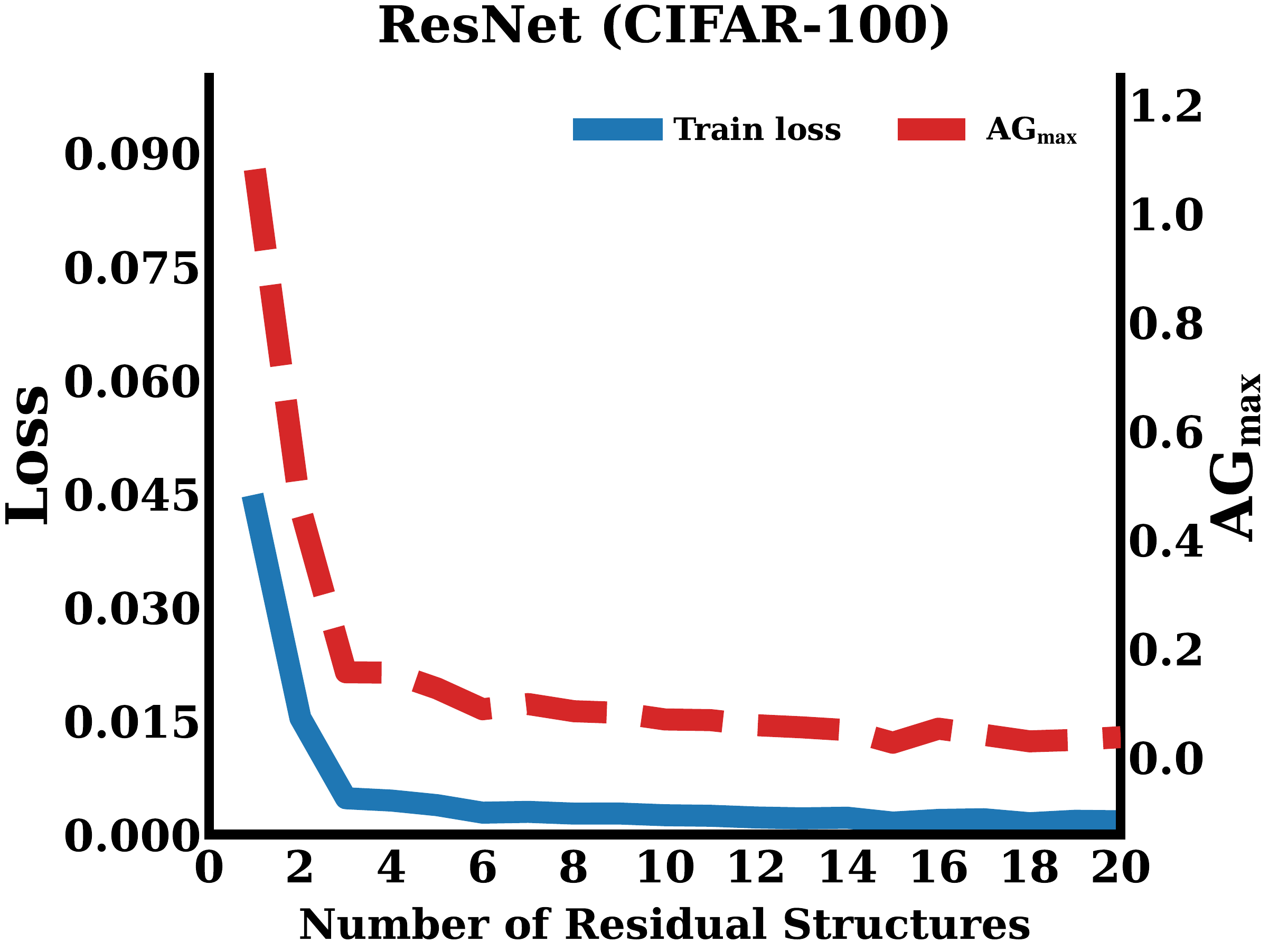}
        \par\smallskip
    \end{minipage}

    \vspace{0.8em}

    \begin{minipage}[t]{0.32\textwidth}
        \centering
        \includegraphics[width=\linewidth]{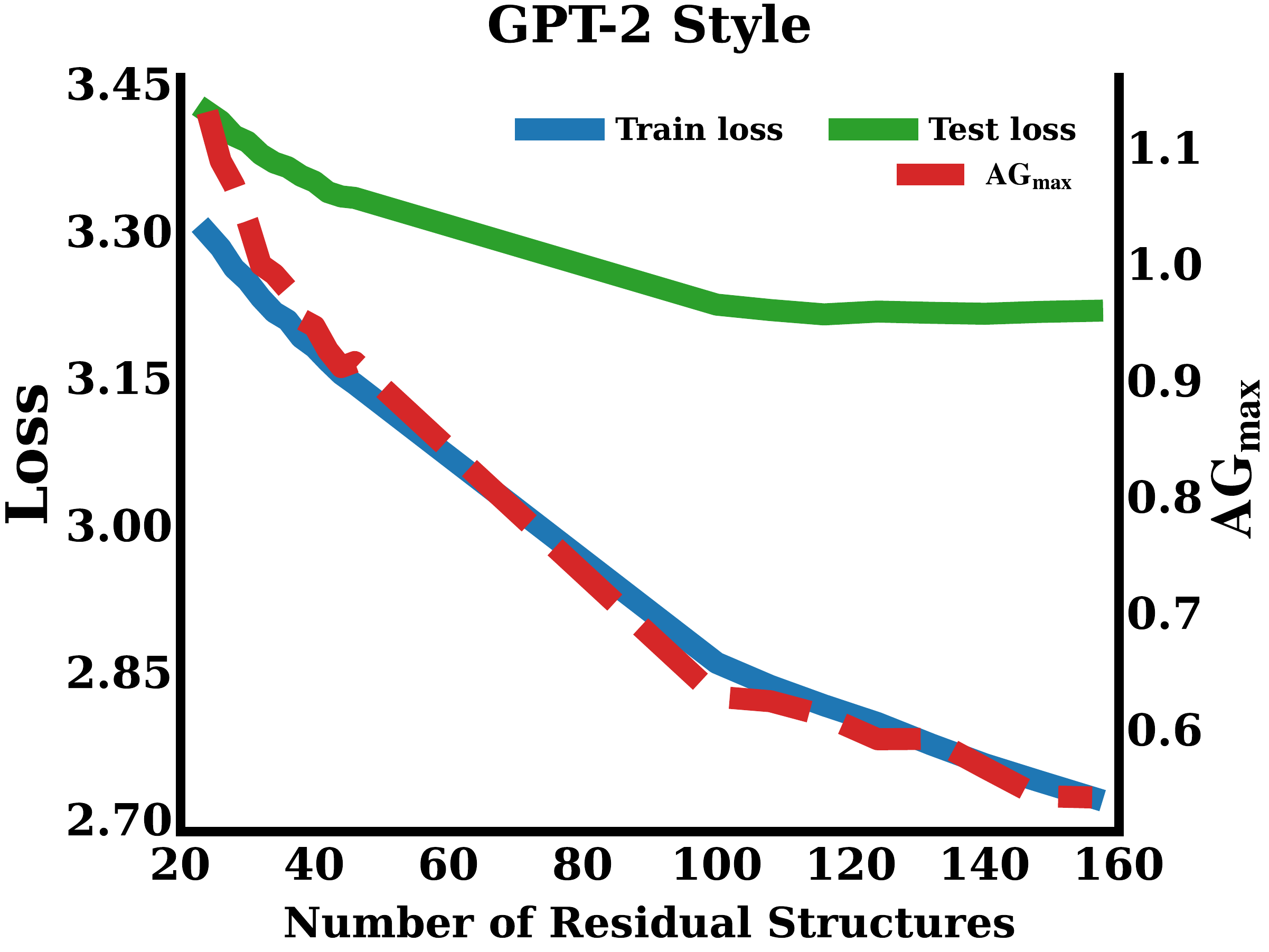}
        \par\smallskip
    \end{minipage}
    \hfill
    \begin{minipage}[t]{0.32\textwidth}
        \centering
        \includegraphics[width=\linewidth]{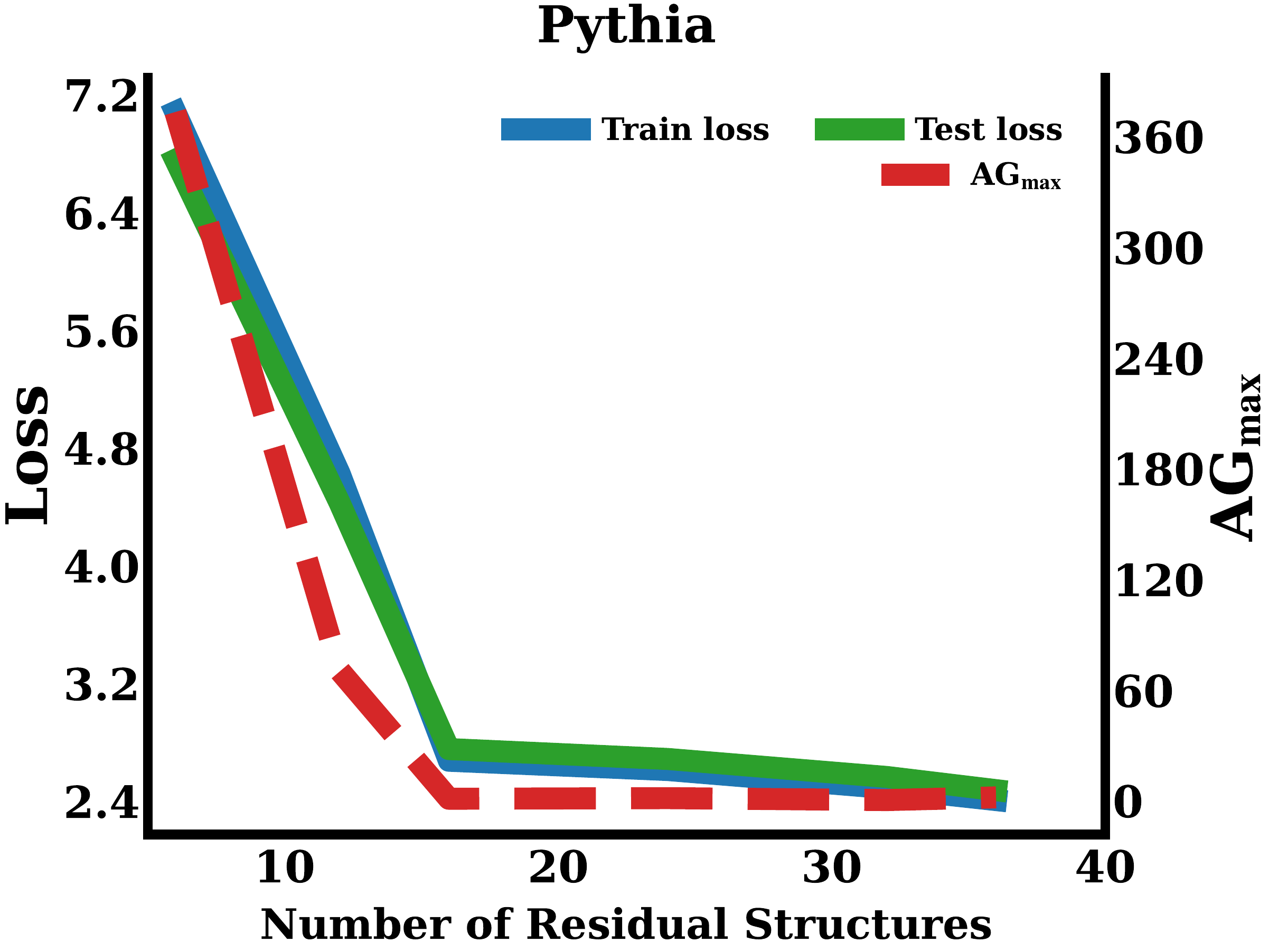}
        \par\smallskip
    \end{minipage}
    \hfill
    \begin{minipage}[t]{0.32\textwidth}
        \centering
        \includegraphics[width=\linewidth]{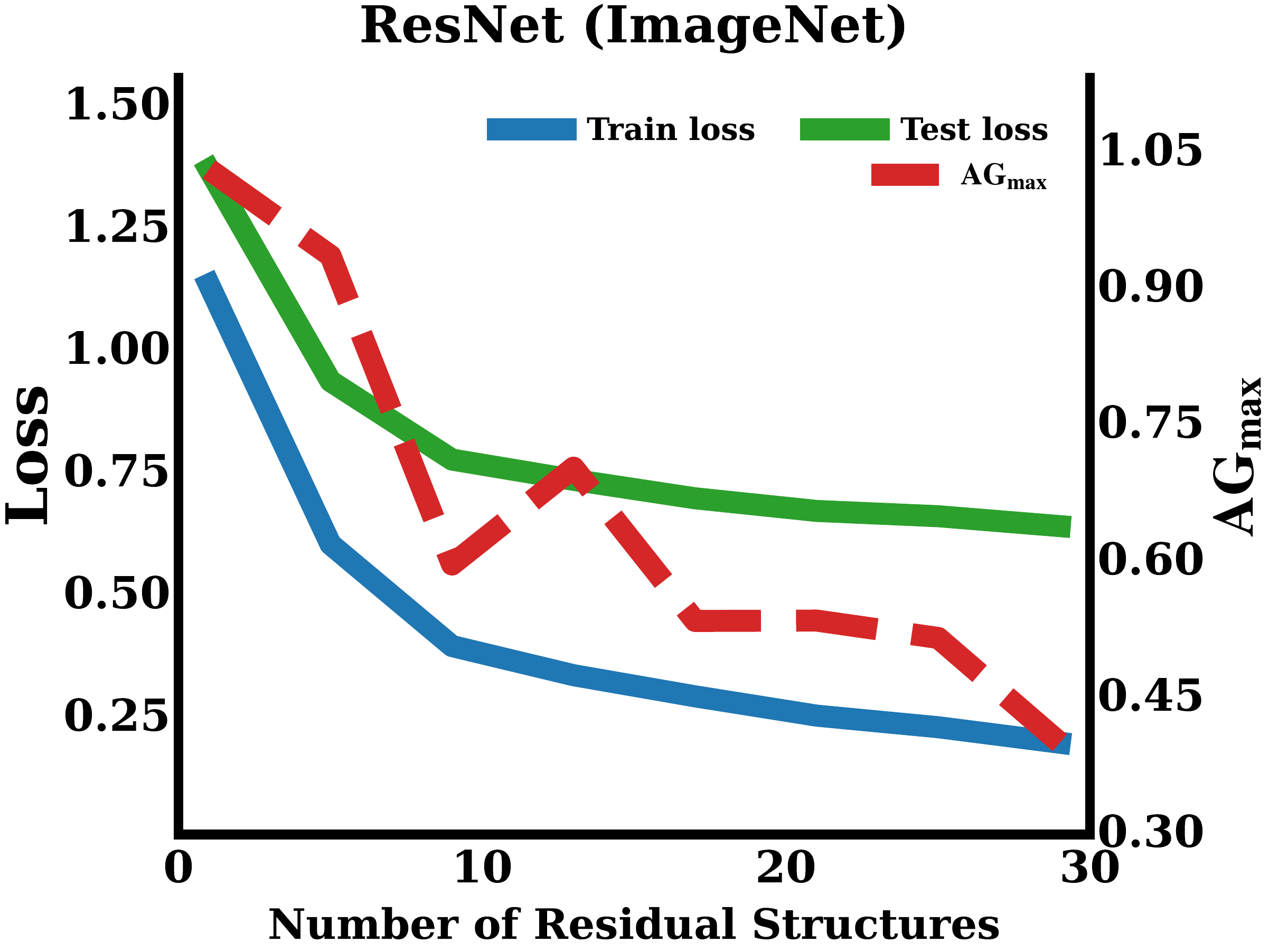}
        \par\smallskip
    \end{minipage}

    \caption{The Loss and Maximum Sample-Wise Activation-Gradient Norm Across Model Depth.}
    \label{fig:six-panel-results}
\end{figure*}

We first examine how the fixed-sample activation-gradient signal changes
as model depth increases. This analysis covers ResNets trained on
CIFAR-10, CIFAR-100, and ImageNet-100, official Pythia checkpoints
continued-pretrained on FineWeb-Edu, and GPT-2-style models trained from
scratch on FineWeb-Edu.  

For each admissible candidate \(l\in\mathcal I\), i.e., at the input of an existing residual structure in this experiment,  define $\mathrm{AG}_{l}
:=
\|q_{S,l}\|_{S,l}
=
\left(
\frac{1}{M}\sum_{i=1}^{M}\|q_{l,i}\|_{2}^{2}
\right)^{1/2},
\mathrm{AG}_{\max}
:=
\max_{l\in\mathcal I}\mathrm{AG}_{l}$.

Under sample-wise residual-signal realizability,
\(q_{S,l}\in\operatorname{Range}(A_{S,l})\), the empirical
activation-gradient signal lies entirely in the admissible residual
tangent space. Hence, $\Pi_{S,l}q_{S,l}=q_{S,l},
\qquad
V_{S,l}
=
\|\Pi_{S,l}q_{S,l}\|_{S,l}
=
\mathrm{AG}_{l}$. If realizability holds at every candidate, then $
V_{S,\mathrm{depth}}
=
\max_{l\in\mathcal I}V_{S,l}
=
\mathrm{AG}_{\max}$. Thus, under sample-wise realizability, \(\mathrm{AG}_{\max}\) exactly
equals the empirical first-order residual-growth value.    Without
verified realizability, we only have $V_{S,\mathrm{depth}}\leq \mathrm{AG}_{\max}$,
so \(\mathrm{AG}_{\max}\) is the upper bound of $V_{S,\mathrm{depth}}$. However, it is enough to use this upper bound to gain the value of $V_{S,\mathrm{depth}}$ because for most of cases, $AG_{\max}$ is close to zero.

\subsubsection{Analysis}
Figure~\ref{fig:six-panel-results} shows a broadly consistent transition 
from a high-signal, depth-beneficial regime to a stable low-signal 
regime. For clearly show resluts, CIFAR 100 results are divideds into two subfigures. For the controlled ResNet and GPT-2-style sweeps, 
$AG_{\max}$ decreases rapidly together with task loss at shallow and 
intermediate depths, and then changes only modestly once the loss 
improvements begin to diminish. Pythia exhibits a similar but sharper 
transition. ImageNet-100 is less monotone at intermediate depths, but 
the overall envelope of $AG_{\max}$ still decreases and reaches its 
minimum at the largest evaluated depth.

The relevant observation is therefore not strict monotonicity at every 
checkpoint, but the emergence of a persistent low-signal plateau. 
Because $AG_{\max}$ is the maximum over all admissible insertion 
locations, a small 
$AG_{\max}$ conservatively implies that the empirical projected 
residual-growth value is small at every candidate. The alignment 
between this regime and diminishing task-loss improvements is 
consistent with the proposed first-order saturation criterion. 
However, a large raw score does not guarantee a realizable residual 
direction, and the training-sample diagnostic need not exactly track 
test loss performance. These results therefore motivate using a stable 
low-signal plateau, rather than a universal threshold or an isolated 
checkpoint, as the practical indicator of residual depth saturation.

As shown in Figure \ref{fig:six-panel-results}, the proposed depth-saturation indicator exhibits a consistent relationship with the marginal benefit of increasing model depth across CIFAR-10, CIFAR-100, and Pythia. At relatively shallow depths, the indicator decreases rapidly as additional residual structures improve model performance. However, once the indicator approaches the low-signal regime—approximately $5*10^{-3}$, further increases in depth yield little or no additional performance improvement. This transition occurs at approximately 16 residual structures for the CIFAR-10 ResNet ($AG_{max}=0.00381$), 11 residual structures for the CIFAR-100 ResNet ($AG_{max}=0.00320$), and 33 residual structures for Pythia ($AG_{max}=0.0008$). The close alignment between the emergence of a near-zero indicator and the disappearance of measurable performance gains supports the use of the proposed metric as a practical diagnostic of residual depth sufficiency. In particular, a persistently small value suggests that little first-order optimization value remains available from adding further residual structures.

 \paragraph{Extended-depth evaluation.}
We conducted independent large-depth experiments (Training hundreds of these models is prohibitively expensive.). For ImageNet-100, ResNet-256/257/258 yielded $AG_{max}$ scores of $0.0020$/$0.0018$/$0.0011$, with training losses decreasing to $0.017$/$0.015$/$0.011$ and test losses fluctuating at $0.66$/$0.72$/$0.59$. Similarly, GPT-2 configurations (700/701/702 residual structures) showed $AG_{max}$ scores of $0.0033$/$0.0029$/$0.0022$, training losses decreasing to $2.67$/$2.62$/$2.55$, and test losses fluctuating at $2.22$/$2.83$/$2.07$. In both settings, $AG_{max}$ scores remained below approximately $5\times10^{-3}$.  The change of training and test loss is so small.  Given the probabilistic nature of test loss performance and its sensitivity to finite-sample and optimization variability, depth saturation need not manifest as monotonically worsening test loss. Crucially, test performance merely fluctuated within a range without systematic gains from increased depth. This confirms that once $AG_{max}$ scores reach a persistent near-zero level, additional residual structures yield minimal marginal improvements in test loss performance, even if isolated deeper checkpoints occasionally achieve better point estimates.

\subsection{Function-Preserving Growth versus From-Scratch Training}
\label{sec:growth-vs-scratch}

For the same four CIFAR-10 ResNet configurations, we compare the
function-preserving grown model with a model of the same final
architecture trained from random initialization. Both models use the
same optimization recipe and are trained to their respective convergence
criteria. This comparison tests whether the disappearance of post-growth
improvement could be explained by an optimization disadvantage caused
by the function-preserving initialization. It is a convergence-quality
comparison rather than a compute-matched estimate of training
efficiency. The model in this section is well-trained and fully converged with enough computational resources. The Local Effect of Activation-Gradient-Matched Insertions experiments is shown in Appendix.

\paragraph{Comparison protocol.}
We evaluate four ResNet configurations on CIFAR-10. For each
configuration, the \emph{Growth} model is obtained by inserting one
globally zero-output residual block into a converged shallower network.
The insertion preserves the reference function exactly before
optimization. After insertion, the parameters of the deeper model are
trained using the classification objective until the prescribed
convergence criterion is reached. The corresponding \emph{Scratch}
model has exactly the same final architecture and parameter count, but
all of its parameters are initialized randomly and trained from scratch.

The two procedures use the same dataset, data preprocessing, model
architecture, optimizer family, regularization, and convergence
criterion. This experiment compares the quality of the solutions
reachable from the two initializations. It is not a compute-matched
comparison of training efficiency, because the Growth model inherits
the optimization already invested in its shallower reference model.

\begin{table}[t]
    \centering
    \caption{Converged losses for training from scratch and
    function-preserving block growth on CIFAR-10. Each row compares the
    same final ResNet architecture under the two initialization
    procedures.}
    \label{tab:resnet_growth}
    \small
    \setlength{\tabcolsep}{3.5pt}
    \begin{tabular}{@{}lcccc@{}}
        \toprule
        & \multicolumn{2}{c}{Train Loss}
        & \multicolumn{2}{c}{Test Loss} \\
        \cmidrule(lr){2-3}
        \cmidrule(lr){4-5}
        Model & Scratch & Growth & Scratch & Growth \\
        \midrule
        ResNet-10 & 0.0088 & 0.0070 & 0.3045 & 0.2589 \\
        ResNet-11 & 0.0082 & 0.0068 & 0.2874 & 0.2626 \\
        ResNet-12 & 0.0059 & 0.0069 & 0.2822 & 0.2580 \\
        ResNet-13 & 0.0059 & 0.0057 & 0.2948 & 0.2510 \\
        \bottomrule
    \end{tabular}
\end{table}

\subsubsection{Analysis}
As shown in Table~\ref{tab:resnet_growth}, function-preserving growth
reaches training losses that are closely comparable to those obtained by
training the same final architectures from scratch. In three of the four
configurations---ResNet-10, ResNet-11, and ResNet-13---Growth achieves a
lower training loss, whereas ResNet-12 shows only a small increase relative
to Scratch. Averaged across all four architectures, the training loss is
$0.0066$ for Growth and $0.0072$ for Scratch. These differences are small
in absolute magnitude and do not indicate a systematic optimization
disadvantage caused by the zero-output initialization. In particular,
preserving the original network function at insertion does not appear to
trap the expanded model near the inherited solution or prevent the enlarged
architecture from reaching a competitive minimum after subsequent task-loss
optimization. Within the evaluated CIFAR-10 ResNet configurations,
function-preserving insertion therefore provides a viable initialization
for continued training.

The test loss results are also consistently favorable to Growth. Growth
obtains a lower test loss for all four final architectures, with absolute
reductions ranging from $0.0242$ for ResNet-12 to $0.0456$ for ResNet-10.
The average test loss decreases from $0.2922$ to $0.2576$, corresponding to
a relative reduction of approximately $11.8\%$. This improvement is not
accompanied by a systematic increase in training loss, suggesting that it
cannot be explained simply by weaker fitting of the training data. The
train--test gap is generally smaller under Growth, especially for
ResNet-12, where Growth has a slightly higher training loss but a lower
test loss. This pattern is consistent with the inherited shallower
representation changing the optimization trajectory of the expanded model
and potentially inducing different implicit regularization.

\section{Conclusion}
We developed a first-order framework for evaluating whether additional
residual depth remains locally useful under a fixed function-preserving
growth protocol. We showed that a residual insertion admits a strict
first-order improvement if and only if the conditional activation gradient
has a nonzero projection onto the corresponding residual tangent space.
For standard zero-output residual blocks, this condition reduces to an
activation--feature cross-gradient criterion.   In the evaluated
CIFAR-10 configurations, function-preserving growth reaches converged
training losses comparable to those obtained from scratch. Together, these
results establish tangent-space projection as an exact local criterion for
the remaining first-order optimization value of residual depth.

\bibliography{aaai2027}

\setcounter{secnumdepth}{2}
\appendix

\section{Complete Notation}
\label{app:complete-notation}

Tables~\ref{tab:notation-model}--\ref{tab:notation-supplementary}
collect the symbols used in the main text and supplementary proofs.
Dependence on the fixed reference model and residual-growth protocol is
suppressed when no ambiguity arises.

\begin{table*}[t]
\centering
\footnotesize
\begin{tabularx}{\textwidth}{@{}lX@{}}
\toprule
Symbol & Meaning \\
\midrule
$\mathcal X,\mathcal Y,\mathcal D$ & Input space, label space, and data distribution on $\mathcal X\times\mathcal Y$. \\
$d_{\mathrm{out}}$ & Predictor output dimension. \\
$S,M$ & A fixed sample $S=\{(x_i,y_i)\}_{i=1}^{M}$ and its cardinality. \\
$\ell,R,\mathcal L_S$ & Loss, population risk, and fixed-sample empirical risk. \\
$f_{\mathrm{old}}^{*}$ & Trained reference model. \\
$\mathcal I,l$ & Fixed finite candidate set and one candidate. \\
$f_{\mathrm{bot}}^{(l)},f_{\mathrm{top}}^{(l)}$ & Reference-model components below and above candidate $l$. \\
$z_l,N_l,\mu_l$ & Hidden state, its dimension, and its induced distribution. \\
$\mathcal F_{\mathrm{res}}^{(l)}$ & Parameterized residual family at candidate $l$. \\
$\Theta_l,p_l,\theta_l$ & Active residual parameter domain, dimension, and local coordinate. \\
$h_{l,\theta_l},f_{l,\theta_l}$ & Residual function and corresponding expanded model. \\
$\Phi_l,\Phi_{S,l}$ & Population and empirical objectives as functions of $\theta_l$. \\
$v_{\mathrm{pop}}^{(l)},v_S^{(l)}$ & Population and empirical directions in the residual non-degeneracy condition stated in the main paper. \\
$d_l,\eta$ & Fixed first-order direction map and free positive step-size scalar. \\
$\bar\eta_l,\bar\eta_{S,l}$ & Population and empirical local descent thresholds. \\
$\epsilon_{\mathrm{Adam}}$ & Positive numerical-stability constant in the first Adam/AdamW direction. \\
$B_l,\rho_l$ & Integrable local Lipschitz envelope and neighborhood radius in the dominated-differentiation remark in the main paper. \\
$\|\cdot\|_{\sigma}$ & Matrix operator norm. \\
\bottomrule
\end{tabularx}
\caption{Model, data, regularity, and optimization notation.}
\label{tab:notation-model}
\end{table*}

\begin{table*}[t]
\centering
\footnotesize
\begin{tabularx}{\textwidth}{@{}lX@{}}
\toprule
Symbol & Meaning \\
\midrule
$q_l,m_l$ & Sample activation gradient and conditional population activation-gradient signal. \\
$J_l$ & Jacobian of $h_{l,\theta_l}$ with respect to $\theta_l$, evaluated at $\theta_l=0$. \\
$g_{\mathrm{pop}}^{(l)},g_S^{(l)}$ & Population and empirical residual-parameter gradients. \\
$D\Phi_l(0)[u],D\Phi_{S,l}(0)[u]$ & Population and empirical Fr\'echet derivatives applied to direction $u$. \\
$\mathcal H_l,A_l,\mathcal T_h^{(l)},\Pi_l$ & Population hidden-state Hilbert space, tangent operator, tangent range, and orthogonal projector. \\
$\mathcal H_{S,l},q_{S,l},Q_S^{(l)}$ & Empirical Hilbert space and tuple/matrix forms of the sample activation-gradient signal. \\
$A_{S,l},\Pi_{S,l}$ & Stacked tangent operator and projector onto its range. \\
$\mathcal V_l,\mathcal V_{S,l}$ & Population and empirical local residual growth values. \\
$\mathcal V_{\mathrm{depth}},\mathcal V_{S,\mathrm{depth}}$ & Maximal local growth values over the fixed candidate set. \\
$A_l^{*},G_l,\widehat G_{S,l}$ & Tangent adjoint, population tangent Gram operator, and empirical tangent Gram matrix. \\
$\operatorname{dist}_{\mathcal H_l}(a,\mathcal T)$ & Distance from $a$ to $\mathcal T$ in the $\mathcal H_l$ norm; $\operatorname{dist}_{S,l}$ denotes the empirical analogue. \\
$\dim(\cdot)$ & Dimension of a finite-dimensional vector space or vectorized tensor. \\
$\varepsilon_{\mathrm{real}}$ & Relative residual-realizability error in $[0,1)$. \\
$L_l,L_{S,l}$ & Optional population and empirical local smoothness constants in Corollary~\ref{cor:local-smoothness-descent}. \\
\bottomrule
\end{tabularx}
\caption{Activation-gradient, tangent-space, and growth-value notation.}
\label{tab:notation-tangent}
\end{table*}

\begin{table*}[t]
\centering
\footnotesize
\begin{tabularx}{\textwidth}{@{}lX@{}}
\toprule
Symbol & Meaning \\
\midrule
$U_l,U_{0,l},V_l,\psi_{l,U_l},\psi_l,r_l$ & Feature parameters, their designated value, active output projection, feature map, fixed feature map, and feature dimension in the standard block. \\
$C_{\mathrm{pop}}^{(l)},C_S^{(l)}$ & Population and empirical activation--feature cross-gradients. \\
$\Psi_S^{(l)},\Pi_{\Psi,l}$ & Residual-feature matrix and projector onto its row space. \\
$\langle\cdot,\cdot\rangle_F,\|\cdot\|_F$ & Frobenius inner product and norm. \\
${}^{\dagger},\operatorname{vec},I_M$ & Moore--Penrose pseudoinverse, vectorization, and $M\times M$ identity. \\
$K,(l_1,\ldots,l_K),F_j$ & Number of jointly inserted blocks, their distinct locations listed from upstream to downstream, and original network segments between them. \\
$T_{l_j,\theta_{l_j}},\iota_j$ & Full residual transformation $w\mapsto w+h_{l_j,\theta_{l_j}}(w)$ and canonical injection into joint block $j$. \\
$\mathcal K,\theta_{\mathcal K},u_{\mathcal K}$ & Ordered joint candidate tuple, product parameter, and joint direction. \\
$f_{\mathcal K,\theta_{\mathcal K}},\Phi_{\mathcal K}$ & Jointly expanded model and its population objective. \\
$S_n^{\mathrm{probe}},\xi_l,\widehat g_n^{(l)}$ & Independent probe sample, one-sample residual-gradient contribution, and probe gradient estimate. \\
$\Gamma,\widehat\Gamma_n,\sigma_l^2$ & Population saturation score, empirical score, and residual-gradient variance. \\
$\tau_n,\widehat{\mathrm{Sat}}_n$ & Detection threshold and empirical saturation decision. \\
$\delta,t,H_0,H_1$ & Practical near-saturation tolerance, separation radius, and hypotheses $H_0:\Gamma\le\delta$ and $H_1:\Gamma\ge\delta+2t$. \\
$\mathbf 1\{\cdot\}$ & Indicator of an event. \\
$\mathsf{AG}_l,\mathsf{AG}_{\max}$ & Experimental scores $\mathsf{AG}_l:=\|q_{S,l}\|_{S,l}=M^{-1/2}\|Q_S^{(l)}\|_F$ and $\mathsf{AG}_{\max}:=\max_{l\in\mathcal I}\mathsf{AG}_l$. \\
$\widetilde{\mathsf{AG}}_l$ & Dimension-normalized visualization score $\mathsf{AG}_l/\sqrt{N_l}$; it is not used to rank candidates with unequal hidden dimensions. \\
\bottomrule
\end{tabularx}
\caption{Standard-block, joint-insertion, supplementary, and experimental notation.}
\label{tab:notation-supplementary}
\end{table*}

\begin{remark}[Normalization of the experimental activation-gradient score]
\label{rem:ag-normalization}
The score $\mathsf{AG}_l$ uses the same sample normalization as
$\|q_{S,l}\|_{S,l}$. Under finite-sample residual-signal realizability,
$\mathsf{AG}_l=\mathcal V_{S,l}$. A coordinate-normalized variant
$\widetilde{\mathsf{AG}}_l=\mathsf{AG}_l/\sqrt{N_l}$ has the same
zero-versus-nonzero boundary but can change candidate rankings when
hidden dimensions differ. It should therefore be treated as a
descriptive visualization scale rather than as the exact empirical
growth value.
\end{remark}
\section{Detailed Experimental Settings}
\label{app:experimental-settings}

This section gives the architecture, data, optimization, checkpoint, and
measurement details for the depth-scaling experiments. Unless stated
otherwise, depth is the controlled architectural variable: width,
classification or language-modeling heads, and the remaining model
hyperparameters are held fixed within each sweep.

\subsection{Depth Variables and Gradient Measurement}
\label{app:exp-depth-score}

\paragraph{ResNet depth.}
The ResNet experiments start from a stagewise ResNet-8 backbone containing
one BasicBlock in each of three stages. We denote by $N$ the number of
residual blocks added to this fixed backbone, so a model contains $3+N$
residual blocks in total. Added blocks are assigned cyclically to
\texttt{layer1}, \texttt{layer2}, and \texttt{layer3}. This rule changes
depth without changing the stage widths or classification head and avoids
placing all added capacity at the end of the network.

\paragraph{GPT-2 depth.}
For GPT-2-style models, $L$ denotes the total number of decoder blocks. The
baseline has $L=12$, and $L=13$ therefore denotes the baseline architecture
with one additional decoder block. Hidden size, feed-forward size, number
of attention heads, positional encoding, normalization, vocabulary, and
the per-update optimization recipe are fixed as $L$ varies; the documented
run durations are specified below.

\paragraph{Recorded activation-gradient statistic.}
For a residual block, decoder block, classifier, or language-modeling head,
let $H_{l,b}$ be the complete input-activation tensor at location $l$ for
analysis batch $b$, and let $\mathcal L_b$ be that batch's loss. Each
analysis batch requires one forward and one backward pass. The run-time
diagnostic stored by the experiment scripts is
\begin{equation}
    \widehat{\mathsf{AG}}_l
    :=
    \frac{1}{B}\sum_{b=1}^{B}
    \left\|
        \nabla_{H_{l,b}}\mathcal L_b
    \right\|_2,
    \label{eq:app-batch-activation-gradient}
\end{equation}
where the norm is taken once over the full, flattened activation-gradient
tensor and the resulting scalar is averaged over $B$ analysis batches.
Inputs to the final fully connected classifier and LM head are recorded as
diagnostic endpoints; the residual-growth candidate set itself contains
the residual- or decoder-block boundaries.

Equation~\eqref{eq:app-batch-activation-gradient} documents the batch-level
quantity retained in the experimental logs. The sample-normalized
Hilbert-space score $\mathsf{AG}_l$ used in the theoretical comparisons is
defined separately in Remark~\ref{rem:ag-normalization}. For a fixed batch
construction, both statistics vanish exactly when all recorded
activation gradients vanish, but their nonzero numerical scales need not
agree. Quantitative comparisons therefore use a fixed analysis-batch
protocol within each model family.

\subsection{CIFAR-10 and CIFAR-100}
\label{app:exp-cifar}

\paragraph{Data and preprocessing.}
CIFAR-10 and CIFAR-100 each contain $50{,}000$ training images and
$10{,}000$ test images at resolution $32\times32$. Training images are
randomly cropped after four-pixel padding and randomly flipped
horizontally. Test images receive no random augmentation. Images are then
normalized using the mean and standard deviation of the corresponding
dataset. CIFAR-10 and CIFAR-100 use 10 and 100 output classes,
respectively.

\paragraph{Architecture.}
Both datasets use the same stagewise ResNet-8 backbone. The base channel
count is 16, and the three stage widths are 16, 32, and 64. Transitions
between stages downsample with stride 2. We sweep $N=1,\ldots,100$, using
the cyclic block-allocation rule described above.

\begin{table*}[t]
\centering
\small
\begin{tabularx}{\textwidth}{@{}lY@{}}
\toprule
Configuration item & CIFAR-10 and CIFAR-100 setting \\
\midrule
Objective and optimizer
    & Cross-entropy loss; SGD with momentum. \\
Initial learning rate
    & $0.1$, multiplied by $0.1$ after epochs 100 and 150. \\
Momentum and weight decay
    & Momentum $0.9$; weight decay $5\times10^{-4}$. \\
Batch size and random seed
    & Batch size 128; seed 1. \\
Maximum duration
    & 300 epochs. \\
Early stopping
    & Enabled only after at least 180 epochs. Training stops when the
      training loss fails to improve by more than $10^{-4}$ for 30
      consecutive epochs. \\
Saved checkpoints
    & The final checkpoint (\texttt{last}) and the checkpoint with the
      lowest training loss (\texttt{best\_train}). Validation and test
      results are not used for checkpoint selection. \\
\bottomrule
\end{tabularx}
\caption{Training protocol for the CIFAR depth sweeps.}
\label{tab:app-cifar-protocol}
\end{table*}

\subsection{ImageNet-100}
\label{app:exp-imagenet100}

\paragraph{Subset construction and preprocessing.}
ImageNet-100 is constructed from ImageNet-1K and contains $128{,}982$
training images and $5{,}000$ validation images, with 50 validation images
per class. Training uses \texttt{RandomResizedCrop(224)} and random
horizontal flipping. At evaluation time, the shorter image side is
resized to 256 pixels and a $224\times224$ center crop is taken. Images
are normalized by the standard ImageNet mean and standard deviation.
The data are loaded with \texttt{torchvision.datasets.ImageFolder}, and
the number of classes is inferred from the directory structure.

\paragraph{Architecture and run configuration.}
To isolate the effects of dataset scale and depth, the ImageNet-100
experiments retain the three-stage CIFAR ResNet-8 backbone, including its
$3\times3$, stride-1 stem. In particular, they do not introduce the
$7\times7$, stride-2 convolution or max pooling used by standard
ImageNet ResNets. The initial sweep uses
$N\in\{1,5,9,\ldots,49\}$, with base channel count 32. Each configuration
is run for 300 epochs with global batch size 128, automatic mixed
precision, and data parallelism across four GPUs.


\subsection{GPT-2 Depth Sweep on Ascend}
\label{app:exp-gpt2}

\paragraph{Model architecture.}
The GPT-2 experiments use the MCore \texttt{GPTModel} implementation in
Megatron-LM/MindSpeed. The $L=12$ baseline follows the main dimensions of
GPT-2 Small: hidden size 768, feed-forward size 3072, 12 attention heads,
and maximum sequence length 1024. It uses learned absolute positional
embeddings, LayerNorm, GELU activations, attention and hidden dropout of
0.1, and initialization standard deviation 0.02. The vocabulary contains
$50{,}257$ tokens, input and output embeddings are tied, and linear-layer
biases are retained.

Only the number of decoder blocks changes across the controlled depth
sweep; all other architectural dimensions are held fixed. The parameter
count therefore increases with depth. The extended-depth configurations
reported in the main paper follow the same architectural conventions.

\paragraph{FineWeb-Edu data.}
The training corpus is a FineWeb-Edu subset containing approximately
$310{,}000{,}347$ GPT-2 tokens. The data are divided into 98\% training,
1\% validation, and 1\% test splits. The training split contains
approximately $296{,}715$ packed sequences of length 1024. With global
batch size 32, we define one \emph{round} as 9,273 optimizer steps,
approximately one complete pass over the training split.

\begin{table*}[t]
\centering
\small
\begin{tabularx}{\textwidth}{@{}lY@{}}
\toprule
Configuration item & GPT-2 depth-sweep setting \\
\midrule
Numerical precision and device allocation
    & BF16; one Ascend NPU per model for the initial controlled-depth runs. \\
Batching
    & Micro-batch size 4, global batch size 32, and eight gradient
      accumulation steps. \\
Optimizer
    & Adam with $\beta_1=0.9$, $\beta_2=0.95$, weight decay 0.1, and
      gradient clipping at 1.0. \\
Learning-rate schedule
    & Initial rate $2.5\times10^{-4}$, minimum rate
      $2.5\times10^{-5}$, 1\% warmup, and cosine decay. \\
Random seed
    & 1234. \\
Training duration
    & The initial controlled-depth runs are capped at five rounds.
  Selected checkpoints are subsequently continued to a total of ten
  rounds. \\
Early stopping
    & The range of the mean training losses over the most recent ten
      rounds must be below $5\times10^{-5}$. With a ten-round cap, this
      criterion can first be evaluated only after the final round. \\
Saved checkpoints
    & The lowest-training-loss checkpoint (\texttt{best\_train}) and the
      final checkpoint (\texttt{last}) at every completed depth. \\
\bottomrule
\end{tabularx}
\caption{Optimization protocol for the documented GPT-2 runs.}
\label{tab:app-gpt2-protocol}
\end{table*}

The completed schedules documented in
Table~\ref{tab:app-gpt2-protocol} cover the controlled-depth experiments.
The extended-depth configurations reported in the main paper use the
same model definition and per-update optimization recipe.


\subsection{Pythia Continued-Pretraining Analysis}
\label{app:exp-pythia}

The Pythia study considers six deduplicated pretrained model
configurations spanning a range of model sizes. Each model is initialized
from the official \texttt{step143000} checkpoint and then continued
pretraining on FineWeb-Edu before activation-gradient measurement.

During measurement, model parameters are frozen while gradients are
retained for hidden activations. Activation-gradient norms are recorded
at decoder-block inputs using the fixed analysis protocol described in
Equation~\eqref{eq:app-batch-activation-gradient}. Because the Pythia
configurations vary in architectural dimensions in addition to depth,
these results provide cross-model evidence rather than a controlled
depth-only comparison.

\section{Preliminary Identities}
\label{app:preliminary-identities}

\begin{lemma}[Population directional derivative]
\label{lem:population-directional-derivative}
Under the first-order regularity assumption stated in the main paper, for every
$u\in\mathbb R^{p_l}$,
\begin{align}
D\Phi_l(0)[u]
&=\mathbb E[q_l(z_l,y)^{\top}J_l(z_l)u] \notag\\
&=(g_{\mathrm{pop}}^{(l)})^{\top}u
=\langle m_l,A_lu\rangle_{\mathcal H_l}.
\label{eq:app-pop-directional}
\end{align}
\end{lemma}

\begin{proof}
For $\mathcal D$-almost every $(x,y)$, the Fr\'echet chain rule gives
\[
D_{\theta_l}\ell(f_{l,0}(x),y)[u]
=q_l(z_l,y)^{\top}J_l(z_l)u.
\]

The first-order regularity assumption stated in the main paper permits
differentiation
under the expectation, giving the first equality in
\eqref{eq:app-pop-directional}. The second follows from the definition
of $g_{\mathrm{pop}}^{(l)}$. Since $J_l(z_l)u$ is measurable with
respect to $z_l$, conditional expectation gives
\[
\mathbb E[q_l^{\top}J_lu]
=\mathbb E[m_l(z_l)^{\top}J_l(z_l)u]
=\langle m_l,A_lu\rangle_{\mathcal H_l}.
\]
\end{proof}

\begin{lemma}[Bounded tangent operator and adjoint identity]
\label{lem:adjoint-projection}
The operator $A_l:\mathbb R^{p_l}\to\mathcal H_l$ is bounded. Its
Hilbert adjoint satisfies
$A_l^{*}m_l=g_{\mathrm{pop}}^{(l)}$. Consequently,
\begin{equation}
g_{\mathrm{pop}}^{(l)}=0
\quad\Longleftrightarrow\quad
\Pi_lm_l=0
\quad\Longleftrightarrow\quad
\mathcal V_l=0.
\label{eq:app-pop-zero-equivalence}
\end{equation}
\end{lemma}

\begin{proof}
For $u\in\mathbb R^{p_l}$,
\[
\|A_lu\|_{\mathcal H_l}^2
=\mathbb E\|J_l(z_l)u\|_2^2
\le\mathbb E\|J_l(z_l)\|_{\sigma}^2\|u\|_2^2,
\]
so $A_l$ is bounded. Lemma~\ref{lem:population-directional-derivative}
shows
$\langle m_l,A_lu\rangle_{\mathcal H_l}
=(g_{\mathrm{pop}}^{(l)})^{\top}u$ for every $u$, hence
$A_l^{*}m_l=g_{\mathrm{pop}}^{(l)}$.
Now $A_l^{*}m_l=0$ exactly when $m_l$ is orthogonal to
$\operatorname{Range}(A_l)=\mathcal T_h^{(l)}$, which is equivalent to
$\Pi_lm_l=0$. The final equivalence is the definition of $\mathcal V_l$.
\end{proof}

\begin{lemma}[Empirical directional derivative and projection]
\label{lem:empirical-directional-derivative}
For every fixed sample satisfying the differentiability conditions in the main paper's first-order regularity assumption and every
$u\in\mathbb R^{p_l}$,

\begin{equation}
D\Phi_{S,l}(0)[u]
=(g_S^{(l)})^{\top}u
=\langle q_{S,l},A_{S,l}u\rangle_{S,l}.
\label{eq:app-emp-directional}
\end{equation}
Moreover,
\begin{equation}
g_S^{(l)}=0
\quad\Longleftrightarrow\quad
\Pi_{S,l}q_{S,l}=0
\quad\Longleftrightarrow\quad
\mathcal V_{S,l}=0.
\label{eq:app-emp-zero-equivalence}
\end{equation}
\end{lemma}

\begin{proof}
The sample-wise chain rule and finite summation give
\[
D\Phi_{S,l}(0)[u]
=\frac1M\sum_{i=1}^{M}q_{l,i}^{\top}J_l(z_{l,i})u
=(g_S^{(l)})^{\top}u.
\]
The same expression is
$\langle q_{S,l},A_{S,l}u\rangle_{S,l}$. The finite-dimensional
orthogonality argument used in Lemma~\ref{lem:adjoint-projection} gives
\eqref{eq:app-emp-zero-equivalence}.
\end{proof}

\section{Proof of the Main-Paper Necessary-and-Sufficient Saturation Theorem}
\label{app:proof-main-theorem}

\begin{proof}
Fix $l\in\mathcal I$ and abbreviate $g=g_{\mathrm{pop}}^{(l)}$.
By Lemma~\ref{lem:population-directional-derivative}, the residual non-degeneracy condition stated in the main paper is the existence of $v$
with $g^{\top}v<0$. This implies $g\neq0$. Conversely, if $g\neq0$,
choosing $v=-g$ gives $g^{\top}v=-\|g\|_2^2<0$; the interior-point
condition on $\Theta_l$ makes sufficiently small steps feasible. Thus
statements 1 and 2 are equivalent. Lemma~\ref{lem:adjoint-projection}
gives the equivalence of statements 2, 3, and 4.

Suppose $g\neq0$. The fixed zero-state descent-compatible update
assumption stated in the main paper and
Lemma~\ref{lem:population-directional-derivative} give
\[
D\Phi_l(0)[d_l(g)]=g^{\top}d_l(g)<0.
\]
Fr\'echet differentiability therefore yields $\bar\eta_l>0$ such that
$\Phi_l(\eta d_l(g))<\Phi_l(0)$ for every
$\eta\in(0,\bar\eta_l)$. Conversely, if $g=0$, zero-state descent
compatibility gives $d_l(g)=0$, so strict decrease is impossible. This
proves statement 5.

Since $\mathcal I$ is finite,
$\mathcal V_{\mathrm{depth}}=0$ if and only if every
$\mathcal V_l=0$. The local equivalences identify this with depth-wide
first-order saturation. Lemma~\ref{lem:empirical-directional-derivative}
gives the empirical result by the same argument, with a threshold
$\bar\eta_{S,l}>0$.
\end{proof}

\begin{corollary}[Quantitative descent under local smoothness]
\label{cor:local-smoothness-descent}
Fix candidate $l$ and suppose ordinary gradient descent is used, so
$d_l(g)=-g$. If $\nabla\Phi_l$ is $L_l$-Lipschitz on a neighborhood
containing the segment
$\{-\eta g_{\mathrm{pop}}^{(l)}:0\le\eta\le1/L_l\}$, then
\[
\Phi_l\!\left(-\eta g_{\mathrm{pop}}^{(l)}\right)
\le \Phi_l(0)
-\eta\left(1-\frac{L_l\eta}{2}\right)
\|g_{\mathrm{pop}}^{(l)}\|_2^2
\]
for every $\eta\in(0,1/L_l]$. In particular, at $\eta=1/L_l$ the
one-step decrease is at least
$\|g_{\mathrm{pop}}^{(l)}\|_2^2/(2L_l)$. The same statement holds for
$\Phi_{S,l}$ with empirical smoothness constant $L_{S,l}$.
\end{corollary}

\begin{proof}
The standard descent lemma gives
$\Phi_l(-\eta g)\le\Phi_l(0)-\eta\|g\|_2^2
+(L_l\eta^2/2)\|g\|_2^2$. Substitute
$g=g_{\mathrm{pop}}^{(l)}$. The empirical proof is identical.
\end{proof}

\subsection{Closed Forms for the Growth Values}
\label{app:growth-value-closed-form}

Define
$G_l:=A_l^{*}A_l=\mathbb E[J_l(z_l)^{\top}J_l(z_l)]$ and
$\widehat G_{S,l}:=M^{-1}\sum_{i=1}^{M}J_l(z_{l,i})^{\top}J_l(z_{l,i})$.

\begin{lemma}[Gram-matrix representation]
\label{lem:growth-value-gram}
The growth values satisfy
\begin{equation}
\mathcal V_l^2
=(g_{\mathrm{pop}}^{(l)})^{\top}G_l^{\dagger}
 g_{\mathrm{pop}}^{(l)},
\qquad
\mathcal V_{S,l}^2
=(g_S^{(l)})^{\top}\widehat G_{S,l}^{\dagger}g_S^{(l)}.
\label{eq:app-growth-value-gram}
\end{equation}
\end{lemma}

\begin{proof}
Because $\mathcal T_h^{(l)}=\operatorname{Range}(A_l)$,
$\Pi_l=A_lG_l^{\dagger}A_l^{*}$. Hence
\[
\mathcal V_l^2
=\langle m_l,A_lG_l^{\dagger}A_l^{*}m_l\rangle_{\mathcal H_l}
=(g_{\mathrm{pop}}^{(l)})^{\top}G_l^{\dagger}g_{\mathrm{pop}}^{(l)}.
\]
The empirical formula follows from the same finite-dimensional
projection identity for $A_{S,l}$ under the empirical inner product.
\end{proof}

\begin{remark}[Parameterization invariance and fixed-sample interpolation]
The geometric definition of $\mathcal V_l$ depends only on
$\operatorname{Range}(A_l)$ and is therefore invariant under every
regular local reparameterization that preserves this range. For an
invertible coordinate change $\theta_l=T\alpha_l$, one has
$A_l\mapsto A_lT$,
$g_{\mathrm{pop}}^{(l)}\mapsto T^{\top}g_{\mathrm{pop}}^{(l)}$, and
$G_l\mapsto T^{\top}G_lT$. Applying
Lemma~\ref{lem:growth-value-gram} in either coordinate system gives the
same scalar $\mathcal V_l^2$, because the represented tangent subspace is
unchanged. By contrast, $\|g_{\mathrm{pop}}^{(l)}\|_2$ is
coordinate-dependent.

Fix a sample size $M$ and draw $S\sim\mathcal D^M$. Suppose the selected
standard block lies in an overparameterized regime in which sample-wise
tangent completeness holds almost surely at this fixed $M$; this
requires at least $r_l\ge M$ and
$\operatorname{rank}(\Psi_S^{(l)})=M$ almost surely. Then
$\mathcal V_{S,l}^2=M^{-1}\sum_i\|q_{l,i}\|_2^2$, and
\[
\mathbb E_{S\sim\mathcal D^M}\mathcal V_{S,l}^2
=\mathbb E\|m_l(z_l)\|_2^2
+\mathbb E\|q_l(z_l,y)-m_l(z_l)\|_2^2
\ge \mathcal V_l^2.
\]
This is a fixed-$M$ interpolation identity, not an assertion that
sample-wise tangent completeness persists as $M\to\infty$ when the
feature dimension $r_l$ is fixed. It explains why the supplementary
consistency result estimates the population residual gradient rather
than asserting
$\mathcal V_{S,\mathrm{depth}}\to\mathcal V_{\mathrm{depth}}$ without
additional rank-stability and complexity assumptions.
\end{remark}

\section{Joint Insertion of Several Residual Blocks}
\label{app:joint-insertion}

Let $\mathcal K=(l_1,\ldots,l_K)$ be an ordered tuple of candidates
at distinct insertion locations, listed from upstream to downstream in
the fixed feed-forward network topology. Write the original network as
$f_{\mathrm{old}}^{*}=F_K\circ F_{K-1}\circ\cdots\circ F_0$, where the
maps $F_j$ are the original network segments between consecutive
insertion locations. Define
$T_{l_j,\theta_{l_j}}(w):=w+h_{l_j,\theta_{l_j}}(w)$ and
\begin{equation}
f_{\mathcal K,\theta_{\mathcal K}}
:=F_K\circ T_{l_K,\theta_{l_K}}\circ F_{K-1}\circ\cdots\circ
T_{l_1,\theta_{l_1}}\circ F_0,
\label{eq:app-joint-model}
\end{equation}
with $K:=|\mathcal K|$,
$\theta_{\mathcal K}:=(\theta_{l_1},\ldots,\theta_{l_K})$ in the product
space $\prod_{j=1}^{K}\mathbb R^{p_{l_j}}$, and
$\Phi_{\mathcal K}(\theta_{\mathcal K})
:=R(f_{\mathcal K,\theta_{\mathcal K}})$.

\begin{corollary}[Joint insertion of independently parameterized blocks]
\label{cor:joint-insertion}
Assume that the first-order regularity assumption stated in the main paper for every
$l\in\mathcal K$. Assume additionally that, for
$\mathcal D$-almost every $(x,y)$, the joint sample loss is Fr\'echet
differentiable with respect to $\theta_{\mathcal K}$ at the joint
origin, and that $\Phi_{\mathcal K}$ is Fr\'echet differentiable there
with derivative obtained by interchanging differentiation and
expectation. Suppose the active parameter blocks are disjoint and every
$h_{l,0}\equiv0$ globally.

A \emph{strict joint first-order descent direction} is a vector
$u_{\mathcal K}$ satisfying
$D\Phi_{\mathcal K}(0)[u_{\mathcal K}]<0$. Such a direction exists if
and only if at least one candidate in $\mathcal K$ satisfies the residual non-degeneracy condition stated in the main paper.
\end{corollary}

\begin{proof}
Let $\iota_j:\mathbb R^{p_{l_j}}\to\prod_{r=1}^{K}\mathbb R^{p_{l_r}}$
be the canonical injection into block $j$. Because every
$h_{l,0}\equiv0$, each full residual transformation held at its
designated origin is
$T_{l,0}(w)=w$ on the full ambient space. Therefore, for every
$\theta_{l_j}$,
\[
\Phi_{\mathcal K}(\iota_j\theta_{l_j})=\Phi_{l_j}(\theta_{l_j});
\]
no differentiability of the intermediate network segments is needed for
this pointwise identity.

Let $L:=D\Phi_{\mathcal K}(0)$, which is a continuous linear functional
by joint Fr\'echet differentiability. Restricting $L$ to coordinate
block $j$ and applying Lemma~\ref{lem:population-directional-derivative}
gives
\[
L[\iota_j u_j]
=D\Phi_{l_j}(0)[u_j]
=(g_{\mathrm{pop}}^{(l_j)})^{\top}u_j.
\]
Every joint direction decomposes as
$u_{\mathcal K}=\sum_{j=1}^{K}\iota_j u_j$, so linearity yields
\[
D\Phi_{\mathcal K}(0)[u_{\mathcal K}]
=\sum_{j=1}^{K}(g_{\mathrm{pop}}^{(l_j)})^{\top}u_j.
\]
Thus the joint derivative is represented by the concatenated vector
$(g_{\mathrm{pop}}^{(l_1)},\ldots,g_{\mathrm{pop}}^{(l_K)})$. It has a
negative direction if and only if at least one component is nonzero,
which is equivalent to the residual non-degeneracy condition stated in the main paper for
at least one candidate by the main paper's necessary-and-sufficient saturation theorem.
\end{proof}

\section{Proof of the Main-Paper Activation-Gradient Characterization Theorem}
\label{app:proof-activation-certificate}

\begin{proof}
If $m_l\in\mathcal T_h^{(l)}$, then $\Pi_lm_l=m_l$, so
$\mathcal V_l=\|m_l\|_{\mathcal H_l}$. The main paper's necessary-and-sufficient saturation theorem gives the population equivalence. If
$q_{S,l}\in\operatorname{Range}(A_{S,l})$, then
$\Pi_{S,l}q_{S,l}=q_{S,l}$, yielding the empirical equivalence.

For the standard block, every matrix direction $\Delta V_l$ satisfies
$J_l(z)\operatorname{vec}(\Delta V_l)=\Delta V_l\psi_l(z)$. Therefore,
\begin{align*}
D\Phi_l(0)[\operatorname{vec}(\Delta V_l)]
&=\mathbb E[q_l(z_l,y)^{\top}\Delta V_l\psi_l(z_l)]\\
&=\langle C_{\mathrm{pop}}^{(l)},\Delta V_l\rangle_F.
\end{align*}
Hence
$g_{\mathrm{pop}}^{(l)}=\operatorname{vec}(C_{\mathrm{pop}}^{(l)})$
under the fixed vectorization convention, and the population residual non-degeneracy condition stated in the main paper holds exactly when
$C_{\mathrm{pop}}^{(l)}\neq0$. The empirical statement follows
identically and does not use residual-signal realizability.

If $U_l$ and $V_l$ are both treated as active raw parameters, then at
$V_l=0$ the derivative with respect to $U_l$ is
$V_lD_{U_l}\psi_{l,U_l}(z)=0$. Hence the full tangent range is identical
to the range obtained by fixing $U_l=U_{0,l}$ and varying only $V_l$;
this reduction loses no first-order directions.

The sample tangent range in matrix form is
$\{\Delta V_l\Psi_S^{(l)}:\Delta V_l\in\mathbb R^{N_l\times r_l}\}$.
Under the tuple--matrix identification,
$\langle B,C\rangle_{S,l}=M^{-1}\langle B,C\rangle_F$; multiplication
by a positive scalar does not change orthogonality or the associated
projector. A matrix belongs to this range exactly when each row belongs
to the row space of $\Psi_S^{(l)}$. Hence the matrix representation of
$\Pi_{S,l}q_{S,l}$ is $Q_S^{(l)}\Pi_{\Psi,l}$ and
$\mathcal V_{S,l}=M^{-1/2}\|Q_S^{(l)}\Pi_{\Psi,l}\|_F$.
Moreover,
$C_S^{(l)}=0$ if and only if
$Q_S^{(l)}\Pi_{\Psi,l}=0$. The stacked tangent operator is surjective
exactly when the row space of $\Psi_S^{(l)}$ is all of $\mathbb R^M$,
equivalently when $\operatorname{rank}(\Psi_S^{(l)})=M$. Then
$\Pi_{\Psi,l}=I_M$ and $C_S^{(l)}=0$ if and only if $Q_S^{(l)}=0$.
Applying the local equivalences at every candidate gives the depth-wide
statement.
\end{proof}

\begin{remark}[Approximate residual-signal realizability]
Orthogonal Pythagoras gives
\begin{align*}
\mathcal V_l^2
&=\|m_l\|_{\mathcal H_l}^2
-\operatorname{dist}_{\mathcal H_l}
  (m_l,\mathcal T_h^{(l)})^2,\\
\mathcal V_{S,l}^2
&=\|q_{S,l}\|_{S,l}^2
-\operatorname{dist}_{S,l}
  (q_{S,l},\operatorname{Range}(A_{S,l}))^2.
\end{align*}
Consequently, if
$\operatorname{dist}_{\mathcal H_l}(m_l,\mathcal T_h^{(l)})
\le\varepsilon_{\mathrm{real}}\|m_l\|_{\mathcal H_l}$ for some
$\varepsilon_{\mathrm{real}}\in[0,1)$, then
$\mathcal V_l\ge\sqrt{1-\varepsilon_{\mathrm{real}}^2}\,
\|m_l\|_{\mathcal H_l}>0$ whenever $m_l\neq0$. The identical
conclusion holds on a fixed sample using the empirical distance and
norm. This approximate form is often more relevant than exact
realizability when the feature dimension is smaller than the sample
size.
\end{remark}

\section{Supplementary Finite-Sample Consistency Result}
\label{app:finite-sample-consistency}

Under the independent finite-sample probing assumption stated in the main paper, instantiate all empirical
quantities with $S=S_n^{\mathrm{probe}}$ and $M=n$. Define
$\xi_l(x,y):=J_l(z_l)^{\top}q_l(z_l,y)$ and
\begin{align*}
\widehat g_n^{(l)}
&:=\frac1n\sum_{i=1}^{n}
\xi_l(\widetilde x_i,\widetilde y_i),\\
\Gamma
&:=\max_{l\in\mathcal I}\|g_{\mathrm{pop}}^{(l)}\|_2,
&
\widehat\Gamma_n
&:=\max_{l\in\mathcal I}\|\widehat g_n^{(l)}\|_2.
\end{align*}
Let
$\sigma_l^2:=\mathbb E\|\xi_l-g_{\mathrm{pop}}^{(l)}\|_2^2$ and, for
a deterministic threshold $\tau_n>0$, define
$\widehat{\mathrm{Sat}}_n
:=\mathbf 1\{\widehat\Gamma_n\le\tau_n\}$.
By the main paper's necessary-and-sufficient saturation theorem,
\begin{align}
\Gamma=0
&\quad\Longleftrightarrow\quad
\mathcal V_{\mathrm{depth}}=0 \notag\\
&\quad\Longleftrightarrow\quad
f_{\mathrm{old}}^{*}\text{ is first-order depth-saturated}.
\label{eq:app-gamma-saturation}
\end{align}

\begin{theorem}[Consistent finite-sample saturation detection]
\label{thm:app-finite-sample-consistency}
Under the independent finite-sample probing assumption stated in the main paper, for every $t>0$,
\begin{equation}
\Pr\!\left(
\max_{l\in\mathcal I}
\|\widehat g_n^{(l)}-g_{\mathrm{pop}}^{(l)}\|_2\ge t
\right)
\le\frac{1}{nt^2}\sum_{l\in\mathcal I}\sigma_l^2.
\label{eq:app-uniform-probe-bound}
\end{equation}
If $\tau_n\to0$ and $n\tau_n^2\to\infty$, then
\begin{align*}
\Gamma=0
&\Longrightarrow
\Pr(\widehat{\mathrm{Sat}}_n=1)\to1,\\
\Gamma>0
&\Longrightarrow
\Pr(\widehat{\mathrm{Sat}}_n=0)\to1.
\end{align*}
In view of \eqref{eq:app-gamma-saturation}, this detector is consistent
for first-order depth saturation relative to the fixed protocol.
\end{theorem}

\begin{proof}
Independence and centering give
$\mathbb E\|\widehat g_n^{(l)}-g_{\mathrm{pop}}^{(l)}\|_2^2
=\sigma_l^2/n$. Markov's inequality applied to the squared norm and a
union bound over the finite set $\mathcal I$ prove
\eqref{eq:app-uniform-probe-bound}.

If $\Gamma=0$, then
\[
\Pr(\widehat{\mathrm{Sat}}_n=0)
=\Pr(\widehat\Gamma_n>\tau_n)
\le\frac{1}{n\tau_n^2}\sum_l\sigma_l^2\to0.
\]
If $\Gamma>0$, choose $l_*$ with
$\|g_{\mathrm{pop}}^{(l_*)}\|_2=\Gamma$. For all sufficiently large
$n$, $\tau_n<\Gamma/2$. On the event
$\widehat\Gamma_n\le\tau_n$,
$\|\widehat g_n^{(l_*)}-g_{\mathrm{pop}}^{(l_*)}\|_2>\Gamma/2$, and
therefore
$\Pr(\widehat{\mathrm{Sat}}_n=1)
\le4\sigma_{l_*}^2/(n\Gamma^2)\to0$.
\end{proof}

\begin{remark}[Coordinate dependence and practical near-saturation]
The statistic $\Gamma=\max_l\|g_{\mathrm{pop}}^{(l)}\|_2$ and the
threshold $\tau_n$ use Euclidean norms of residual-parameter gradients.
Their zero-versus-nonzero boundary is invariant under regular
reparameterization, but their numerical scales are protocol- and
coordinate-dependent. The consistency theorem should therefore be
interpreted relative to the fixed parameterization.

Exact saturation, $\Gamma=0$, is a knife-edge population hypothesis. A
practical separated test may instead fix $\delta\ge0$ and $t>0$ and
compare $H_0:\Gamma\le\delta$ with
$H_1:\Gamma\ge\delta+2t$ using decision threshold $\delta+t$. On the
event
$\max_l\|\widehat g_n^{(l)}-g_{\mathrm{pop}}^{(l)}\|_2<t$, the test is
correct under either hypothesis. Thus
\eqref{eq:app-uniform-probe-bound} bounds its error probability by
$\sum_l\sigma_l^2/(nt^2)$. Implementing the rule requires an upper
estimate of the variances $\sigma_l^2$; rate-optimal threshold selection
and adaptive power analysis require additional assumptions.
\end{remark}

\begin{remark}[Why an independent probe sample is needed]
The fixed-sample equivalences in the main theorems are deterministic and
may be evaluated on the training sample. The consistency theorem instead
interprets an empirical gradient as an estimator of its population
counterpart. Reusing the data that trained the reference model generally
breaks the conditional i.i.d. argument unless additional stability or
sample-splitting assumptions are introduced.
\end{remark}

\section{More experiments}
On four converged CIFAR-10 ResNets, we insert a globally zero-output residual block and fit only its terminal projection to a small negative activation-gradient target, testing whether the resulting first-order direction produces an immediate reduction in the empirical task loss.

\subsection{Local Effect of Activation-Gradient-Matched Insertions}
\label{sec:gradient-matched-insertion-results}

We next examine whether an activation-gradient-matched residual block
produces the immediate task-loss change predicted by the fixed-sample
first-order analysis. We evaluate four converged CIFAR-10 ResNet
checkpoints with reference depths 8, 14, 20, and 34.

For four converged ResNet configurations on CIFAR-10, we directly
construct one additional residual block by matching its output to a
small negative activation-gradient perturbation. At the selected
candidate location $l$, we freeze the reference model and cache the
training-set hidden representations $z_{l,i}$ together with their
activation gradients $q_{l,i}$. Both quantities are detached and treated
as a fixed regression dataset. The inserted block is a standard ResNet block. Its feature-producing parameters
$U_l$ are initialized using the standard ResNet initialization and then
held fixed, while the terminal output projection is initialized at
$V_l=0$ and is the only optimized parameter. We fit $V_l$ by gradient
descent on
\begin{equation*}
\mathcal{J}_{\mathrm{match}}^{(l)}(V_l)
:=
\frac{1}{2M}
\sum_{i=1}^{M}
\left\|V_l\psi_l(z_{l,i})
+
\alpha q_{l,i}
\right\|_2^2,
\end{equation*}
where the plus sign reflects the regression target
$-\alpha q_{l,i}$. The zero initialization makes the insertion globally
function-preserving before fitting, while training only $V_l$ restricts
the construction to the first-order tangent family generated by the
fixed residual features. After the auxiliary objective has converged,
the fitted block is inserted and evaluated immediately; no
classification-loss optimization or joint fine-tuning is performed.
The matching scale, optimization schedule, and candidate-selection rule
are fixed without using held-out performance. We report the matching
error, the relative magnitude of the resulting hidden-state
perturbation, and the changes in training  loss.
 
\begin{table}[t]
    \centering
    \caption{Immediate loss changes after inserting one
    activation-gradient-matched residual block. The fitted model is
    evaluated without subsequent classification-loss training.
    Positive gain denotes lower loss after insertion.}
    \label{tab:gradient_matched_insertion}
    \small
    \setlength{\tabcolsep}{2.3pt}
    \scalebox{0.85}{
        \begin{tabular}{@{}lcccc@{}}
            \toprule
            Reference
            & Train loss(ave)
            & $\Delta_{\mathrm{train}}^{\mathrm{match}}$
            & Test loss(ave)
            & $\Delta_{\mathrm{test}}^{\mathrm{match}}$ \\
            \midrule
            ResNet-8
            & $0.0427\!\rightarrow\!0.0420$
            & \makecell{$+0.0007$ \\ $\pm 0.0003$}
            & $0.3782\!\rightarrow\!0.3779$
            & \makecell{$+0.0003$ \\ $\pm 0.0009$} \\
            \midrule 
            ResNet-14
            & $0.0392\!\rightarrow\!0.0388$
            & \makecell{$+0.0004$ \\ $\pm 0.0002$}
            & $0.2537\!\rightarrow\!0.2531$
            & \makecell{$+0.0006$ \\ $\pm 0.0011$} \\
            \midrule
            ResNet-20
            & $0.0251\!\rightarrow\!0.0249$
            & \makecell{$+0.0002$ \\ $\pm <0.0001$}
            & $0.2015\!\rightarrow\!0.2017$
            & \makecell{$-0.0002$ \\ $\pm 0.0005$} \\
            \midrule
            ResNet-34
            & $0.0011\!\rightarrow\!0.0009$
            & \makecell{$+0.0002$ \\ $\pm <0.0001$}
            & $0.1891\!\rightarrow\!0.1895$
            & \makecell{$-0.0004$ \\ $\pm 0.0005$} \\
            \bottomrule
        \end{tabular}
    }
\end{table}
 
\subsubsection{Analysis.}
Table~\ref{tab:gradient_matched_insertion} reports the immediate loss changes
after inserting one activation-gradient-matched residual block, without
subsequent task-loss optimization or joint fine-tuning. The insertion
reduces the training loss for all four checkpoints, by
$7\times10^{-4}$, $4\times10^{-4}$, $2\times10^{-4}$, and
$2\times10^{-4}$ for ResNet-8, ResNet-14, ResNet-20, and ResNet-34.
Because the target is constructed from activation gradients on the same
training sample, this sign consistency is the main behavior predicted by
the fixed-sample first-order analysis. The results are consistent with the
fitted block capturing a descending component within the tangent family
generated by the fixed features.

The reductions are small, as expected from $\alpha=10^{-3}$ and from
evaluation before task-loss fine-tuning. Thus, the experiment tests an
immediate local improvement rather than the gain after fully optimizing
the expanded model. The reduction decreases from $7\times10^{-4}$ for
ResNet-8 to $2\times10^{-4}$ for ResNet-20 and ResNet-34, qualitatively
consistent with diminishing empirical first-order value of additional
depth. However, differing baseline losses and one insertion per
checkpoint prevent Table~\ref{tab:gradient_matched_insertion} from establishing a
monotone or calibrated depth--gain relationship. Test-loss point estimates improve by
$3\times10^{-4}$ and $6\times10^{-4}$ for ResNet-8 and ResNet-14, but
worsen by $2\times10^{-4}$ and $4\times10^{-4}$ for ResNet-20 and
ResNet-34. All changes are smaller than or comparable to the reported
uncertainty.
\end{document}